\documentclass[twoside,11pt]{article}

\usepackage[nohyperref, preprint]{jmlr2e}
\usepackage{microtype}
\usepackage{graphicx}
\usepackage{mathrsfs}
\usepackage{booktabs}
\usepackage{float}
\usepackage{array}
\usepackage[hidelinks]{hyperref}
\usepackage{comment}
\usepackage{mathtools}
\usepackage{color}
\usepackage{wrapfig}
\newcolumntype{C}[1]{>{\centering\arraybackslash}m{#1}}
\newcolumntype{L}[1]{m{#1}}

\newcommand{\sectionname}{Section}
\newcommand{\secref}[1]{\sectionname~\ref{#1}}

\newcommand{\figref}[1]{\figurename~\ref{#1}}
\newcommand{\tabref}[1]{\tablename~\ref{#1}}
\DeclareMathOperator{\Unif}{Unif}
\newcommand{\dd}{\mathrm{d}}

\newcommand{\kB}{k_{\mathrm{B}}}
\newcommand{\unit}[1]{\mathrm{#1}}
\newcommand{\kg}{\unit{kg}}
\newcommand{\g}{\unit{g}}
\newcommand{\m}{\unit{m}}
\newcommand{\cm}{\unit{cm}}
\newcommand{\mm}{\unit{mm}}
\renewcommand{\l}{\unit{\ell}}
\newcommand{\s}{\unit{s}}
\renewcommand{\d}{\unit{d}}
\newcommand{\K}{\unit{K}}
\newcommand{\A}{\unit{A}}
\newcommand{\J}{\unit{J}}
\newcommand{\N}{\unit{N}}
\newcommand{\Pa}{\unit{Pa}}
\newcommand{\V}{\unit{V}}
\newcommand{\sxrightarrow}[2][]{%
  \mathrel{\text{$\xrightarrow[#1]{#2}$}}%
}

\title{ Dimensionless machine learning: \\ Imposing exact units equivariance}

\author{\name Soledad Villar \email soledad.villar@jhu.edu \\
       \addr Department of Applied Mathematics and Statistics, Johns Hopkins University\\
       \addr Mathematical Institute for Data Science, Johns Hopkins University
       \AND
       \name Weichi Yao \email weichi.yao@nyu.edu\\
       \addr Department of Technology, Operations and Statistics, New York University
       \AND
       \name David W. Hogg \email david.hogg@nyu.edu \\
       \addr Center for Cosmology and Particle Physics, Department of Physics, New York University\\
       \addr Max-Planck-Institut f\"ur Astronomie\\
       \addr Flatiron Institute, a Division of the Simons Foundation
       \AND
       \name Ben Blum-Smith \email bblumsm1@jhu.edu \\
       \addr Department of Applied Mathematics and Statistics, Johns Hopkins University
       \AND
       \name Bianca Dumitrascu \email bmd39@cam.ac.uk\\
       \addr Department of Computer Science and Technology, Cambridge University
       }
\editor{TBD}
\jmlrheading{1}{2000}{1-48}{4/00}{10/00}{meila00a}{Villar et al.}

\ShortHeadings{Units equivariance for machine learning}{Villar et al.}
\firstpageno{1}

\begin{document}

\maketitle
\thispagestyle{empty}

\begin{abstract}\noindent
Units equivariance (or units covariance) is the exact symmetry that follows from the requirement that relationships among measured quantities of physics relevance must obey self-consistent dimensional scalings. Here, we express this symmetry in terms of a (non-compact) group action, and we employ dimensional analysis and ideas from equivariant machine learning to provide a methodology for exactly units-equivariant machine learning: For any given learning task, we first construct a dimensionless version of its inputs using classic results from dimensional analysis, and then perform inference in the dimensionless space. Our approach can be used to impose units equivariance across a broad range of machine learning methods which are equivariant to rotations and other groups. We discuss the in-sample and out-of-sample prediction accuracy gains one can obtain in contexts like symbolic regression and emulation, where symmetry is important. We illustrate our approach with simple numerical examples involving dynamical systems in physics and ecology.
\end{abstract}

\section{Introduction}\label{sec:intro}
In recent years, there has been enormous progress in developing machine learning methods that incorporate exact symmetries.
Indeed, the revolutionary performance of convolutional neural networks \citep{lecun1989backpropagation} was enabled by the imposition of a local translational symmetry at the bottom layer of the networks.
When a learning problem obeys an exact symmetry, we have the strong intuition that imposing the symmetry on the method must improve learning and generalization.
Much of the work on exact symmetries has been situated in physics and natural-science domains \citep{yu-physics,kashinath2021physics}, because physical laws exactly obey a panoply of symmetries.

One of the symmetries of physics---and indeed all of the sciences---is the symmetry underlying dimensional analysis.
Quantities that are added or subtracted must have identical units -- if the units system of the inputs to a function is changed, the units of the output must change accordingly.
This symmetry---which we call ``units equivariance'' but in physics it might be more natural to call this units covariance---is a passive symmetry that applies to all problems (see Section 4.1 of \cite{rovelli2000loop}).
It has many implications.
One is that it is possible to derive scalings and dependencies of outputs on inputs from the units directly.
For example, if a problem involves only a length $L$, an acceleration $g$, and a mass $m$, and the problem is to learn or predict a time $t$, it is possible from units alone to see that $t\propto\sqrt{L/g}$.
Another implication is that inhomogeneous functions, such as transcendental functions and most non-linear functions, can only be applied to dimensionless (unitless) quantities:
if a quantity $x$ is dimensional, an inhomogeneous polynomial expression in $x$ is inconsistent with the principle that quantities to be added or subtracted must have identical units. 
These dimensional symmetries are strict and exact, and apply to essentially every problem in the sciences.
In chemistry, ecology, and economics, for example, the inputs and the outputs of functional relationships have non-trivial units, and the results must be equivariant to the choice of units system. 
Machine learning methods for physical sciences are often designed in such a way that implies units-equivariance, even if they don't explicitly say so (see for instance \cite{2019ApJ, 2021PSJ}).

In this work, we implement a particular case of group-equivariant machine learning, for the group corresponding to changes of units.
We make use of dimensionless quantites, which are the invariants of this group.
We thereby build on our previous work in which group invariants are used to build group-equivariant functions \citep{villar, yao, blum2022equivariant}.
Our procedure here builds \emph{dimensionless features} out of the problem inputs and then ensures that the resulting outputs are scaled back to their correct dimensions or units.
The guiding philosophy of the work is to transform the inputs into invariant features before they are used to train a machine learning method, and then to un-transform label predictions at the output or at test time.
These approaches to symmetry-respecting machine learning are simple to implement and perform well \citep{yao, chense}.

In what follows, we will make the strong assumption that the dimensions and units of all regression inputs are known, complete, and self consistent.
However, a different direction of research is to look at how dimensional relationships or other symmetries are \emph{discovered}.
This is the setting for \citet{constantine2017data}, and \citet{evangelou2021parameter}.
This idea---discovery of dimensional structure---connects to prior work as a particular case of the more general problem of learning symmetries from data \citep{benton2020learning, cahill2020lie, portilheiro2022tradeoff}.

\paragraph{Related work:}
Dimensional analysis is a classical theory with applications in engineering and science \citep{barenblatt_1996}.
These ideas have been connected to machine learning previously \citep{rudolph1998context, frisone2019buckingham, bakarji2022dimensionally}.
The key theoretical result in dimensional analysis is the Buckingham Pi Theorem \citep{buckingham1914pi}, which says that a function is units equivariant if and only if it is a function of a set of dimensionless quantities.
These quantities can (usually) be obtained as products of integer powers of the input features.
Integer linear-algebra algorithms permit discovery of a generating set of dimensionless features \citep{hubert2012rational}.

Incorporating group invariances and equivariances in the design of neural networks has led to advances in many applications from molecular dynamics, to turbulence, to climate and traffic prediction \citep{batzner2021se,wang2020towards,bakarji2022dimensionally, kashinath2021physics, jin2020composing}. 

In many applications, symmetries are implemented approximately via data augmentation \citep{baird1992document, van2001art, wong2016understanding, cubuk2018autoaugment, dao2019kernel, chen2020group, shen2022data}. In other applications symmetries are implemented exactly.
In graph neural networks, the learned functions are equivariant with respect to actions by permutations of the order in which nodes appear in the adjacency matrix \citep{gilmer2017neural,duvenaud2015convolutional, chen2019cdsbm,gama2020graphs}.
Parametrizing such functions efficiently and universally is a difficult task because of its connection with the graph isomorphism problem. Many methods and theoretical results have been proposed to address this challenge  \citep{xu2018powerful,morris2019higher,chen2019equivalence,chen2020can, huang2021short}.

More generally, in equivariant machine learning, neural networks are restricted to only represent functions which are invariant or equivariant with respect to group actions \citep{cohen2016group, maron2018invariant, kondor2018n}.
Some of these methods involve group convolutions~\citep{cohen2016group,Cohen2016steerable,wang2020certified}, irreducible representations of groups \citep{fuchs2020se,kondor2018n,thomas2018tensor,weiler20183d,cohen2018spherical,Weiler2019e2equivariant}, or constraints on optimization \citep{finzi}. 
Others involve the construction of explicitly invariant features \citep{gripaios2021lorentz, haddadin2021invariant, villar, blum2022equivariant}.
We are proponents of the approach of constructing explicit invariants.

In regression problems, a recent line of theoretical work shows that imposing symmetries and group equivariance can reduce generalization error in linear \citep{elesedy2021provably} and kernel \citep{elesedy_kernel} settings, as well as sample complexity in (finite) group-invariant kernel settings \citet{bietti2021sample, mei2021learning}.
Most importantly, without imposing certain symmetries many learning algorithms are not able to learn correctly \cite{brugiapaglia2021invariance}.

\paragraph{Our contributions:}\nopagebreak\begin{itemize}
\item 
We provide a definition for \emph{units equivariance} as an equivariance with respect to a group action, and incorporate it into machine learning models, aided by ideas from classical dimensional analysis.
\item
We show that exact units equivariance is easy to impose on many kinds of learning tasks, by constructing a dimensionless version of the learning task, performing the dimensionless task, and then scaling back in the proper dimensions (and units) at the end (perhaps prior to evaluating the cost function).
Dimensionless quantities are invariants with respect to changes of units and can be computed using discrete linear algebra algorithms.
In this sense, the approach we advocate here is related to approaches based on group invariants to impose exact group equivariances \citep{villar}.
\item
We discuss extensions of theoretical results on generalization bounds for regression problems under symmetries generated by compact groups \citep{elesedy2021provably}, to the group of scalings (which is not compact but reductive).  
\item
We demonstrate with a few simple numerical regression problems that the reduction of model capacity (at fixed complexity) delivered by the units equivariance leads to improvements in generalization (in-distribution and out-of-distribution).
In this context, we discuss symbolic regression and emulator related tasks. We also discuss limitations of our approach in the context of unknown dimensional constants.
\end{itemize}

\section{Units, dimensions, and units equivariance}
Almost every physical quantity (any position, velocity, or energy, say) has units (inches, kilometers per hour, or BTUs, say).
In the physical sciences we are advised to use SI units \citep{si}, which include meters ($\m$), meters per second ($\m\,\s^{-1}$), and Joules ($\J$), for example.
Any energy can be converted to $\J$, any velocity can be converted to $\m\,\s^{-1}$, and so on, according to known conversion factors.
There are dimensionless quantities in physics too, such as Reynolds numbers, or concentrations, but these can also be thought of as having units of unity (or percent or parts per million or so on).

Abstracting slightly, all the (say) SI units are built on \emph{base units} of kilograms ($\kg$), meters ($\m$), seconds ($\s$), kelvin ($\K$), amperes ($\A$), and a few others (such as moles).
That is, it is possible to convert any SI unit into powers of the base units.
For example, a pascal ($\Pa$) is a $\kg\,\m^{-1}\,\s^{-2}$, and a volt ($\V$) is a $\kg\,\m^{2}\,\s^{-3}\,\A^{-1}$.
That is, the units of any physical quantity can be converted to powers of the SI base units.

Abstracting even further, there is a concept of dimensions, which is the physical entity being measured by units, or the thing that is unchanged when you change the units of something.
Two energies, even if measured in different unit systems, both have the \emph{dimensions} of energy.
In this sense, there are not just the \emph{base units} of $\kg$, $\m$, $\s$, $\K$, $\A$, and so on; there are also the \emph{base dimensions} of mass, length, time, temperature, current, and so on.

It is always possible to create, from any dimensioned quantity, a dimensionless quantity by multiplying and dividing by powers of quantities with the base dimensions.
Similarly, it is possible to convert any dimensional relationship into a dimensionless relationship.
This operation is formalized in the Buckingham Pi Theorem \citep{buckingham1914pi}, which motivates this work. 

Not only it is the case that physical quantities that are added must have the same \emph{dimensions};
in detail, the \emph{numerical} addition of the quantities must be performed only after they have been converted to the same \emph{units} as well.
And the products (or ratios) of quantities of units will have the products (or ratios) of the input-quantity units.
Connected to this is a concept of consistent base units or \emph{coherence}:
For example, imagine having a mass $M$ measured in grams, a length $L$ measured in inches, a force $F$ measured in newtons, and a speed $V$ measured in kilometers per hour. These quantities have inconsistent base units, in the sense that the time unit inside the force unit is not the same as the time unit inside the speed unit.
Naive multiplication of different combinations of these quantities will produce outputs with incommensurate units.
They have to be converted to consistent units prior to any arithmetic manipulations.
In what follows, we will assume that the inputs and outputs of any model or problem under consideration has been converted into coherent units, for example the explicitly coherent SI system.

Coherence is a hard requirement, but still leaves a lot of room to maneuver.
For example, in one of the examples in \secref{sec:experiments}, we measure horizontal distances in meters and volumes of water in liters.
These are incoherent technically, since a volume can be expressed as a length cubed.
However, since we express the problem such that horizontal distances and volumes never inter-convert, we can coherently express the problem with this choice.

We consider spaces $\mathcal Z = \prod_{i=1}^d \mathcal X_{\mathbf u_i}$, which consist of coherent (meaning described in a consistent units system) dimensional elements. Each factor $\mathcal X_{\mathbf u_i}$ is a space of values for a given feature, which is measured in units specified by a parameter $\mathbf u_i$. Thus $x_1\in \mathcal X_{\mathbf u_1}$ might be a mass, $x_2\in \mathcal X_{\mathbf u_2}$ a temperature, $x_3\in \mathcal X_{\mathbf u_3}$ a velocity, etc. As a set, each $\mathcal X_{\mathbf u_i}$ is just $\mathbb R$, but the specification of its dimensions via $\mathbf u_i$ endows it with a specific action by a group $G$ of rescalings. A precise development of this setup follows.

We fix a list of $k$ base units in terms of which all the desired features can be described. For example, if the features consist of energies, temperatures, velocities, forces, masses, and accelerations, the base units could be $(\kg, \m, \s, \K)$ (and $k=4$). The choice of base units determines a {\em rescaling group} $G := \mathbb R_{>0}^k$. An element $(g_1,\dots,g_k)\in G$ rescales the $i$th base unit by a factor of $g_i$ for each $i$.

Say there are $d$ features. Then for each $i=1,\dots, d$, we express the units of the $i$th feature in terms of the base units, and record this expression in an integer vector $\mathbf u_i\in \mathbb Z^k$, whose $j$th component is the exponent to which the $j$th base unit occurs in this expression. Continuing the example, if the first feature is an energy measured in Joules, then $\mathbf u_1 = [1,2,-2,0]$, because $1\,\J = 1\, \kg\,\m^2/\s^2$. We then define $\mathcal X_{\mathbf u_i}$ to be the real line $\mathbb R$ equipped with the action by $G$ induced by its action on the base units. Explicitly, if $x_i\in \mathcal X_{\mathbf u_i}$ and $g=(g_1,\dots,g_k)\in G$, then the action is given by the formula
\begin{equation} \label{eq.action.factor}
g\cdot x_i = \left(\prod_{j=1}^k g_j^{-u_{ij}}\right) x_i.
\end{equation}
In our running example, if we replace $\kg$ with $\g$ and $\m$ with $\cm$, leaving $\s$ and $\K$ untouched, then the group element that accomplishes this rescaling is $g=(0.001,0.01,1,1)$, and if $x_1=2.9$, representing a value of $2.9\, \J$, then 
\[
g\cdot x_1 = (0.001)^{-1}(0.01)^{-2}(1)^{2}(1)^0 (2.9) = 2.9 \times 10^7
\]
reflecting the fact that $2.9\, \J = 2.9\times 10^7\, \g\,\cm^2/s^2$.

We call the space of features $\mathcal Z$ a \textbf{units-typed space}. It is the cartesian product
\begin{equation} \label{eq.Z}
    \mathcal Z = \prod_{i=1}^d \mathcal X_{\mathbf u_i}.
\end{equation}
It is a real vector space under coordinatewise addition and multiplication by (dimensionless) scalars.  Furthermore, the elements of $\mathcal X_{\mathbf u}$ can be multiplied by elements of $\mathcal X_{\mathbf u'}$. In summary the algebraic rules for $ x \in \mathcal X_{\mathbf u}$ and $ x' \in \mathcal X_{\mathbf u'}$, $\alpha \in \mathbb R$ and $\gamma \in \mathbb Z$ are the following:  
\begin{align}
 \alpha \, (x, \mathbf u) &=(\alpha\, x, \mathbf u), \text{where $\alpha$ is a dimensionless scalar}\\
    (x, \mathbf u) + (x', \mathbf u' ) &= \left \{
    \begin{array}{l}
    (x+x', \mathbf u) \text{ if } \mathbf u=\mathbf u' \\
    \text{does not exist otherwise}
    \end{array}
    \right. \\
    (x, \mathbf u)\, (x', \mathbf u' ) &=(x\, x', \mathbf u + \mathbf u') \\
    (x, \mathbf u)^{\gamma} &= (x^\gamma, \gamma\,\mathbf u ), \text{ where $\gamma$ is a (dimensionless)  integer}.
\end{align}
Thus $\mathcal X_\mathbf{u}$ can be seen as a homogeneous component of degree $\mathbf u$ in a $\mathbb Z^k$-graded algebra. We do not make explicit use of this, however.

Because each factor $\mathcal X_\mathbf{u}$ carries an action \eqref{eq.action.factor} by the rescaling group $G$, the space $\mathcal Z$ does as well. Note that the action is completely specified by the ordered list $(\mathbf u_1,\dots, \mathbf u_d)\in (\mathbb Z^k)^d$ of units vectors. Sometimes we will denote $\mathbf x \in \mathcal Z$ as $\mathbf x=(x_i, \mathbf u_i)_{i=1,\ldots, d}$ to indicate the units of each of its features. 

We call \textbf{units-typed function} any function between units-typed spaces. Definition \ref{def.equiv} defines units equivariant functions.

\begin{definition} \label{def.equiv}
A units-typed function is any function $f: \mathcal Z_\mathcal X \rightarrow \mathcal Z_\mathcal Y$. If in addition it has the property that
\begin{equation}
    f(g\cdot x) = g\cdot (f(x))
\end{equation}
for every $g$ in the rescaling group $G$, then we say that $f$ is a \textbf{units-equivariant function}.
\end{definition}

For example, imagine that $f$ is a units-equivariant function that takes as input a mass, a length, and an acceleration, and returns a time and an energy.
If the mass, length, and acceleration inputs are multiplied by $1$, $1$, and $1/25$ respectively, then the time and energy outputs should end up multiplied by $5$ and $1/25$ respectively.
That is, the units equivariance is an expression of the natural dimensional and units-conversion scalings.

In order for a non-constant function $f$ to be units-equivariant, it will be necessary that each of the vectors of units of $\mathcal Z_{\mathcal Y}$ lies in the span of the set of input vectors of units of $\mathcal Z_{\mathcal X}$.

In other work \citep{villar} we are concerned with geometric (coordinate-free) properties of scalar, vector, and tensor inputs to machine learning problems.
Each component of a vector or a tensor with units (such as a velocity or a stress) can be represented as an element of $\mathcal X$.
In physics all elements of a vector or tensor must have the same units vector $\mathbf u$. By combining the formulation from \citep{villar} with our current formulation, we can make units-equivariant and coordinate-free models (see \secref{sec:experiments}).

\section{Units-equivariant regressions} \label{sec:approach}

Given training data $(\mathbf x_t, \mathbf y_t)_{t=1,\ldots N}$, where $\mathbf x_t\in \mathcal Z_X$ and $\mathbf y_t\in \mathcal Z_Y$ (both units-typed spaces), a units-equivariant regression is a regression restricted to a space of units-equivariant functions.

There are multiple approaches for imposing exact symmetries on machine learning methods. Here we take an invariant-features approach \citep{villar}.
We begin by algorithmically constructing a featurizer $\phi(\cdot)$ that constructs dimensionless features $\xi$ from the dimensioned input data $\mathbf x$, and a decoder $g_{\mathbf x, \mathbf v}(\cdot)$ that converts dimensionless label predictions $\hat{\eta}$ into dimensioned training-label predictions $\hat{\mathbf y}$ with dimensions of $\mathbf v$.
See \figref{fig:approach} for a visualization of the setup.

Regression (or classification, or any other task) proceeds as usual, but in the space of purely dimensionless features and labels, with the dimensioned training labels appearing only in the cost function.
The concept underlying this approach is that, for the units equivariance to be exact, it is necessary that the inputs to any method that implements nonlinear functions of the input $\mathbf x$ (such as a multilayer perceptron, or a kernel regression) act instead only on dimensionless features $\xi$, because otherwise the nonlinearities effectively (internally) add and subtract quantities with different dimensions, which violates the equivariance. This approach borrows ideas from the Buckingham Pi Theorem \citep{buckingham1914pi}, and uses technology from linear algebra over the integers \citep{stanley2016smithnormalform} to construct the featurizer $\phi(\cdot)$ algorithmically. 

\begin{figure}[tp]
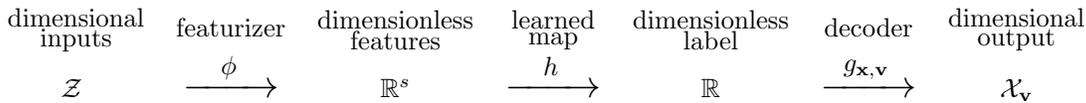

    \centering \small
    \begin{tabular}{ccccccc}
    \hspace{-0.8cm}
         $\let\scriptstyle\textstyle\substack{\text{dimensional}\\ \text{inputs}}$ &
         featurizer &
         $\let\scriptstyle\textstyle\substack{\text{dimensionless}\\ \text{features}}$ &
         $\let\scriptstyle\textstyle\substack{\text{learned}\\ \text{map}}$ &
         $\let\scriptstyle\textstyle\substack{\text{dimensionless}\\ \text{label}}$ &
         decoder & $\let\scriptstyle\textstyle\substack{\text{dimensional}\\ \text{output}}$ 
         \\
         \hspace{-.8cm}$\mathcal Z$
         & 
         ${\Large \sxrightarrow[\hspace*{1cm}]{\phi}}$ &  $\mathbb R^s$ & ${\Large \sxrightarrow[\hspace*{1cm}]{h}}$ & $\mathbb R$ & ${\Large \sxrightarrow[\hspace*{1cm}]{g_{\mathbf x, \mathbf v}}}$ & $\mathcal X_{\mathbf v}$ 
    \end{tabular}
    \caption{Overview of the general approach.}
    \label{fig:approach}
\end{figure}

Specifically, the input space $\mathcal Z$ includes $d$ dimensional input features, such as mass, temperature, etc \eqref{eq.Z}. Some of the inputs will be fundamental constants, such as Newton's constant, speed of light, etc. 
The featurizer $\phi: \mathcal Z \to \mathbb R^s$ delivers $s$ (dimensionless) products of integer powers of the numerical elements of $\mathcal Z$.
If $\mathbf x=(x_i, \mathbf u_i)_{i=1, \ldots, d}$ then $\phi(\mathbf x) =(\phi_1(\mathbf x), \ldots, \phi_s(\mathbf x))$ where for all $j=1,\ldots s$ we have 
\begin{equation} \label{b.pi}
\xi_j=\phi_j(\mathbf x)= \prod_{i=1}^{d} x_i^{\alpha_{ji}} \text{ where }  \sum_{i=1}^{d} \alpha_{ji} \mathbf u_i =\mathbf 0 \in \mathbb Z^k,
\end{equation}
where the constraint guarantees that $\xi_j$ is dimensionless.
The exponents $\alpha_{ji} \in \mathbb Z$ can be found by solving the system of diophantine linear equations in \eqref{b.pi} and the solutions form a lattice, the dimension of which can be computed as:
\begin{equation}
\small
    \text{$\#$dimensionless features = $\#$input variables - $\#$independent units},
\end{equation}
where the number of \emph{independent} units coincides with the number of linearly independent vectors in $\{\mathbf u_i\}_{i=1}^d$. For example, consider three velocities $v_1,v_2,v_3$ with units $\m\,\s^{-1}$, the number of units is two ($\m $ and $\s$), but the number of independent units is one, making the dimensionless scalars a two-dimensional lattice. 

We could select our featurizer to produce a basis of the lattice (we use the Smith Normal Form to this end \citep{stanley2016smithnormalform}, a similar approach to \citep{hubert2012rational}, which uses the Hermite Normal Form), or we could select our featurizer to produce all lattice points within a bounded region. 

If the dimensioned output is $\mathbf y = (y, \mathbf v) \in \mathcal X$ we find an integer solution $\alpha_{yi}$ of $\sum_{i=1}^{d} \alpha_{yi} \mathbf u_i =\mathbf v$ and the decoder $g_{\mathbf x, \mathbf v}:\mathbb R \to \mathcal X$ is $g_{\mathbf x, \mathbf v}(\hat\eta) = \hat\eta \prod_{i=1}^{d} x_i^{\alpha_{yi}}$.
In words, the decoder finds from $\mathbf x$ a product of integer powers of elements of $\mathbf x$ that has the same dimensions as the output label $\mathbf y$, and multiplies the dimensionless label prediction $\hat{\eta}$ by that product.
This is possible because any non-constant units-equivariant function must have the vector $\mathbf v$ lies in the span of the vectors $\{\mathbf u_i\}_{i=1}^d$. 
In practice there is not a unique choice for $\alpha_{yi}$. One practical solution is to learn different models for different solutions $\alpha_{yi}$ and use an ensemble method (see Section \ref{sec:experiments}). 

The training of a regression model usually involves optimization of a loss function.
This is often a norm of a difference between the training labels $\mathbf y$ and their predictions $\hat{\mathbf{y}}$.
When the problem is made dimensionless, this loss can be made dimensionless as well, or else
it can be left dimensional.
In many contexts the loss is a chi-squared objective or a log-likelihood, or can be interpreted as such.
In these cases the loss is dimensionless naturally.
The units-equivariance approach we recommend is agnostic to whether the loss is dimensionless or dimensional; adoption of this approach does not require adoption of any particular loss.

This approach does, however, place several significant burdens on the user:
All elements of the input $\mathbf x$ and output $\mathbf y$ of the method must have well-defined and known units, and the units information must be encoded in the form of a $k$-vector of integer powers of a well-defined list of $k$ base units.
Also, sometimes quantities external to the natural training data must be included with the inputs $\mathbf x$, such as parameters or fundamental constants (such as Newton's constant and the speed of light), that are relevant to unit conversions or natural relationships among quantities with different units.

The key idea is that the dimensionless featurizer $\phi(\cdot)$, which produces $s$ dimensionless features, can be compared with an arbitrary featurizer $\tilde{\phi}(\cdot)$, which produces $p$ features that are not necessarily dimensionless. For example, we can consider the space spanned by rational monomials of inputs of a certain degree. 
\begin{definition}[Rational monomial]
A function $P:\mathbb R^d \to \mathbb R$ is a \textbf{rational monomial} (also known as a \textbf{Laurent monomial}) if 
\begin{equation} \label{eq.rational.monomial}
    P(x_1,\ldots, x_d) = \prod_{i=1}^d a_i x_i^{\alpha_i},
\end{equation}
where the coefficients $a_i\in \mathbb R$ and the exponents $\alpha_i \in \mathbb Z$. The degree of $P$ is defined as $\sum_{i=1}^d |\alpha_i|$.
\end{definition}

The units-equivariance approach \eqref{b.pi} will produce features consisting of dimensionless rational monomials (example: $\text{mass}*\text{spring constant}*(\text{length})^2*(\text{momentum})^{-2}$ is a dimensionless rational monomial), whereas a non-equivariant approach can produce arbitrary features of the same type (rational monomials of bounded degrees). We claim that imposing the units equivariance provides a good inductive bias for the learning problem. In \secref{sec:bias} we briefly discuss the generalization gains of imposing this exact symmetry.

This approach plays well with many machine learning approaches, including linear regression, kernel regression, and deep learning:
It involves only a swap-in replacement for the featurizing or feature normalization that is usually done prior to training a model, and a tweak to the output layer of the method.
Thus it can be adapted to almost any machine learning methods in use at the present day.

\section{Generalization improvements of units equivariance}
\label{sec:bias}

Imposing the exact units-equivariance symmetry in regression tasks improves the prediction performance even in cases where the held out data has out-of-distribution properties.
This empirical improvement can be partly attributed to the fact that the symmetry acts like a physics-informed prior which constraints the regression inputs to satisfy correct dimensional scaling relationships. 
Another reason could be due to dimensionality reduction: when the task is made dimensionless the resulting number of independent, dimensionless inputs is always strictly less than the number of dimensional inputs. 
One way to make this argument precise is by assuming that the units-equivariant functions are a subset of the baseline hypothesis class. 
In this case, at fixed model complexity (e.g. rational monomials of a certain degree), the number of free parameters or the model capacity goes down when the problem is made dimensionless.

In Appendix \ref{sec:analysis} we formally discuss the generalization improvements of certain equivariant models and how these results translate to units-equivariance.  We based our analysis on recent results that show how to compute explicitly the gains in terms of generalization gaps and sample complexity of imposing group invariances and equivariances in machine learning. The results from \citet{mei2021learning} and \citet{bietti2021sample} hold for finite groups, whereas the results from \citet{elesedy2021provably} and \citet{elesedy_kernel} hold for general compact groups.
Specifically, \citet{elesedy2021provably} shows that if we are aiming to learn a target invariant function $f^*$ from samples from an invariant distribution $\mu$, then for any estimator $\hat f$, the (Reynolds) projection of $\hat f$ onto the space of invariant functions has smaller expected error. We explain the details in Appendix \ref{sec:analysis}.  

The problem we consider here, units equivariance, uses the group of scalings, which is unfortunately non-compact so the results from \citet{elesedy2021provably} don't directly apply.
In particular, given a function, it is unclear what is the \emph{correct} notion of its projection onto the space of invariant functions.
This is particularly relevant because our approach does not compute a regression to later project its estimator onto the space of invariant functions. It directly solves a regression in a space of invariant functions. The question of how much does one gain by using this approach assumes the existence of a certain non-invariant baseline.

Given a space of functions $\mathcal F$ and $\bar {\mathcal F}$ the subspace of invariant functions, then a projection onto the space of invariant functions is just an operator $P:\mathcal F \to \bar{\mathcal F}$ that fixes $\bar {\mathcal F}$ pointwise. Therefore the way to project a function onto the space of invariant functions is far from being unique. 
There are two notions of projection we consider, (1) the so-called Reynolds projection, which generalizes the notion of averaging the function along the group orbit, (2) the orthogonal projection with respect to the measure $\mu$ used to generate the data.
In the case of compact groups and data sampled from invariant measures both notions coincide, giving a very simple expression for the generalization gap of a function in comparison with its projection onto the space of invariant functions. However, we show that in the non-compact case, the notions do not coincide for any real measure $\mu$.

In Appendix \ref{sec:analysis}:
\begin{itemize}
\item We summarize the ideas from \citet{elesedy2021provably} that allow them to compute a generalization gap for compact groups. In a nutshell: We assume we are learning an unknown $G$-invariant function $f^*:\mathbb R^d \to \mathbb R$ from samples. We assume that $G$ is a compact group and the data is sampled from a $G$-invariant measure in $\mathbb R^d$ ($x\sim \mu \in \mathbb R^d$, $\mu(U) = \mu(g\cdot U)$ for all $g\in G$ and measurable $U\subset\mathbb R^d$).
The risk of a function $f:\mathbb R^{d}\to \mathbb R$ is defined as the expected test error: 
\begin{equation}
    \mathcal R(f):= \mathbb E_{x\sim \mu} \|f(x) - f^*(x)\|^2
\end{equation}
Elesedy and Zaidy show that by projecting $f$ to the space of invariant functions, the generalization gap always improves. In particular:
\begin{equation}
    \Delta(f,\bar f):= \mathcal R(f) - \mathcal R(\bar f) = \|f^\perp \|_\mu^2 \geq 0
\end{equation}
where $\bar f$ is the (orthogonal or Reynolds) projection of $f$ onto the space of invariant functions (see Appendix \ref{sec:analysis} for the precise definition), and $f^\perp$ is its orthogonal complement in $\mathscr L_2(\mathbb R^d, \mu)$. 

\item We explain how to extend these ideas to non-compact groups using Weyl's trick.
\item We show that if we replace $\mathbb R^d$ with a complex domain in $\mathbb C^d$, and the measure is invariant with respect to complex scalings of modulus $1$, then the results from \citet{elesedy2021provably} apply and we can prove a non-negative generalization gap for equivariant functions with respect to complex scalings.

\item We show that if the domain is real, the two notions of projection (orthogonal and Reynolds) don't match for any choice of measure.
\end{itemize}

\paragraph{Out-of-distribution generalization}
The generalization improvements from \cite{elesedy2021provably, elesedy_kernel, bietti2021sample} assume that the data is sampled from group invariant measures. 
However, since the group of scalings is not compact, scaling invariant probability measures do not exists.  
The exact scaling symmetries enforced by the units equivariance extend to arbitrary dilations of the base dimensions. 
Thus they connect function outputs for very different inputs.
In the context of regression this means that the predictions of trained units-equivariant models ought enable accurate predictions far outside the domain of the training set.
One way to state this is to compare the domains of the dimensionless features to the domains of the raw dimensional features.

Mathematically, this can be stated in terms of out-of-distribution generalization \citep{arjovsky2020out, geisa2021towards, dey2022out}. 
In particular, consider a generic problem with $d$ inputs with dimensions made from $k$ base units, and (therefore) a basis of $s=d-k$ dimensionless quantities. Consider a training set sampled from a distribution $\mu$ supported in a compact set $\mathcal D \subset \mathbb R^d$. This induces a distribution $\phi(\mu)$ in the space of dimensionless features $\mathbb R^s$.
Many distributions in $\mathbb R^d$ (possibly even supported in disjoint sets) map to $\phi(\mu)$.
For a trivial example, consider $d=2$ mass inputs, such that $k=1$; if the training set has masses in $\kg$ drawn from $\Unif(0,1)$ and the test set has masses in $\g$ drawn from $\Unif(0,10^4)$, they will nonetheless both have the same distribution in the one ($d-k=1$) available independent dimensionless quantity (the ratio of masses).
This is one reason why this method allows us to generalize to settings that can be considered out-of-distribution in the original input space.
We conjecture that the out-of-distribution generalization improvement spans beyond this trivial case.
We show numerical evidence of this claim in \secref{sec:experiments}.
This observation is corroborated in related literature where different types of equivariances are shown related to data augmentation procedures \citep{chen2020group}, which, in turn, lead to empirically favorable performance in out-of-distribution settings \citep{liang2022metashift}.
We believe that a rigorous out-of-distribution result could be formalized using techniques from domain adaptation and meta-learning \citep{ben2010theory, mansour2009domain, hanneke2019value, li2018learning, kang2018transferable};
we leave that for future work.


\section{Experimental demonstrations}\label{sec:experiments}

\paragraph{Symbolic regression: Simple springy pendulum}
In this example, we consider a pendulum bob of mass $m$ (units of $\kg$) at the end of a linear spring, swinging under the influence of gravity.
The total mechanical energy or hamiltonian $H$ (units of $\kg\,\m^{2}\,\s^{-2}$) of this system consists of a kinetic energy and two potential-energy contributions:
\begin{align}
    H = \underbrace{\frac{1}{2}\,\frac{|\mathbf{p}|^2}{m}}_{\text{kinetic energy}}
    \underbrace{+\,\frac{1}{2}\,k_\text{s}\,(|\mathbf{q}| - L)^2}_{\substack{\text{spring} \\ \text{potential energy}}}
    \underbrace{-\,m\,\mathbf{g}^\top\mathbf{q}}_{\substack{\text{gravitational} \\ \text{potential energy}}} ~,
\end{align}
where $\mathbf{p}$ is the 3-vector momentum of the bob (units of $\kg\,\m\,\s^{-1}$), $|\mathbf{p}|^2 = \mathbf{p}^\top\mathbf{p}$, $k_\text{s}$ is the spring constant (units of $\N\,\m^{-1}=\kg\,\s^{-2}$), $\mathbf{q}$ is the 3-vector position of the bob relative to the pivot (units of $\m$), $|\mathbf{q}|=\sqrt{\mathbf{q}^\top\mathbf{q}}$, $L$ is the natural length of the spring (units of $\m$), $\mathbf{g}$ is the 3-vector acceleration due to gravity (units of $\m\,\s^{-2}$).
The natural base units here are the SI base units $(\kg,\m,\s)$, but they could just as easily be (stone, furlong, fortnight).

This is almost the simplest possible physics problem.
We honor its simplicity by constructing an extremely simplifed symbolic regression:
Given samples of the parameters $m,k_s,L,\mathbf{g}$, the initial conditions $\mathbf{p},\mathbf{q}$, and the corresponding values of the hamiltonian $H$, we show that (as expected) we can infer the exact functional form of the hamiltonian, and that, imposing units-equivariance signficantly reduces the complexity and prediciton error of the problem.

First we observe that, in Newtonian mechanics, the hamiltonian---or total mechanical energy---is a scalar.
Thus the hamiltonian can be a function only of scalars and scalar products of the vector and scalar inputs \citep{villar}.
We construct all rational scalar monomials of the inputs up to a well-defined degree, including, for example, $m\,k_s\,|\mathbf{q}|\,(\mathbf{g}^\top \mathbf{p})^{-2}$, where the vectors $\mathbf{g},\mathbf{p},\mathbf{q}$ are implicitly column vectors, $|\mathbf{q}|$ is the magnitude of $\mathbf{q}$, and $\mathbf{g}^\top \mathbf{p}$ is the scalar (inner) product of $\mathbf{g}$ and $\mathbf{p}$.
For our purposes, the degree of the rational monomial is the maximum absolute value exponent appearing in the expression, so the example would have degree $2$.
Then we construct all \emph{dimensionless} rational scalar monomials of the inputs up to the same well-defined degree, including, for example, $m\,k_s\,L^2\,|\mathbf{p}|^{-2}$, which is also of degree $2$ but dimensionless.
It turns out that there are far fewer dimensionless rational scalar monomials than rational scalar monomials to any degree.

In detail, the dimensional scalar inputs to our monomial lists are the $9$ scalars $m$, $k_s$, $L$, $|\mathbf{g}|$, $(\mathbf{g}^\top \mathbf{p})$, $(\mathbf{g}^\top \mathbf{q})$, $|\mathbf{p}|$, $(\mathbf{p}^\top \mathbf{q})$, $|\mathbf{q}|$.
We produce all monomials to maximum degree $2$ but subject to two additional rules:
While we count scalars $|\mathbf{g}|$, $|\mathbf{q}|$, and $|\mathbf{p}|$ as having degree $1$, we count scalars $\mathbf{g}^\top \mathbf{p}$ and $\mathbf{p}^\top \mathbf{q}$ and so on as having degree $2$ (so at maximum degree $2$, say, they cannot appear squared).
We also did not permit the dot products $\mathbf{g}^\top \mathbf{p}$ and $\mathbf{p}^\top \mathbf{q}$ and so on to appear with negative powers (because these inverses can produce unbounded singular values in the design matrix).
With these inputs and these rules, there are $286$ dimensionless monomials to (only!) degree $2$, and $187\,500$ total monomials (irrespective of dimensions) to degree $2$.

Given the enormous difference between $286$ and $187\,500$, it is obvious that units equivariance is incredibly informative.
We demonstrate the value experimentally by performing symbolic regressions for the hamiltonian $H$ with two different objectives; one an L2, and the other a LASSO objective.
In each case, the objective is the norm of the difference between the predicted and true hamiltonian value, but made dimensionless by dividing by the quantity $k_s\,L^2$, which has units of energy.
In the L2 case, 8192 training-set objects are used, and in the LASSO case, 128.
Training-set objects are drawn from distributions in $m, k_s, L, \mathbf{g}, \mathbf{p}, \mathbf{q}$, in which the scalars $m, k_s, L$ are drawn from uniforms and the vectors $\mathbf{g}, \mathbf{p}, \mathbf{q}$ are drawn from isotropically oriented unit vectors times magnitudes drawn from uniforms.
In both cases, the regression applied to the dimensionless-feature design matrix delivers machine-precision-level errors on held-out data, and linear-fit coefficients that represent the correct formula or expression for the hamiltonian.

The large number of baseline monomials ($187\,500$) makes it computationally difficult to perform equivalent baseline comparisons. 
This alone demonstrates the value of the units-equivariant approach for symbolic-regression-like problems.
However, in order to test this, we augment the $286$ dimensionless monomials with $500$ randomly chosen dimensional monomials and---at these training-set sizes---the symbolic regressions fail:
They deliver an order of magnitude worse mean-squared error on held-out test data and they do not find the correct coefficients for the hamiltonian expression.
The code is publicly available at Google Colab.\footnote{See \url{https://dwh.gg/springy}.}

\paragraph{Emulator: Springy double pendulum}
Next we consider the task of learning the dynamics of the springy double pendulum, which is a pair of single springy pendula connected with a free pivot\footnote{The source code is published in \url{https://github.com/weichiyao/ScalarEMLP/tree/dimensionless}.} (Figure~\ref{fig:dp_system}).
The goal here is to predict its trajectory at later times from different initial states.
For this task, for each of the realization of the $N$ training data, $m_1, m_2, k_{s1}, k_{s2}, L_1, L_2$ are randomly generated from $\Unif(1,2)$, as well as the norm of the gravitational acceleration vector $\mathbf{g}$. Initializations at $t_0$ of the pendulum positions and momenta are generated as those in \citet{finzi} and \citet{yao}. 
The training labels are the positions and momenta at a set of $\tilde{T}$ later times $t\in\{t_1,\ldots,t_{\tilde{T}}\}$:
\begin{equation}
\mathbf{z}(t)=(\mathbf{q}_1(t),\mathbf{q}_2(t),\mathbf{p}_1(t),\mathbf{p}_2(t)), \quad t\in\{t_0, \ldots,t_{\tilde{T}}\}. \label{eq.training}
\end{equation}  

\begin{figure}[t]
    \centering
    \includegraphics[scale=0.15]{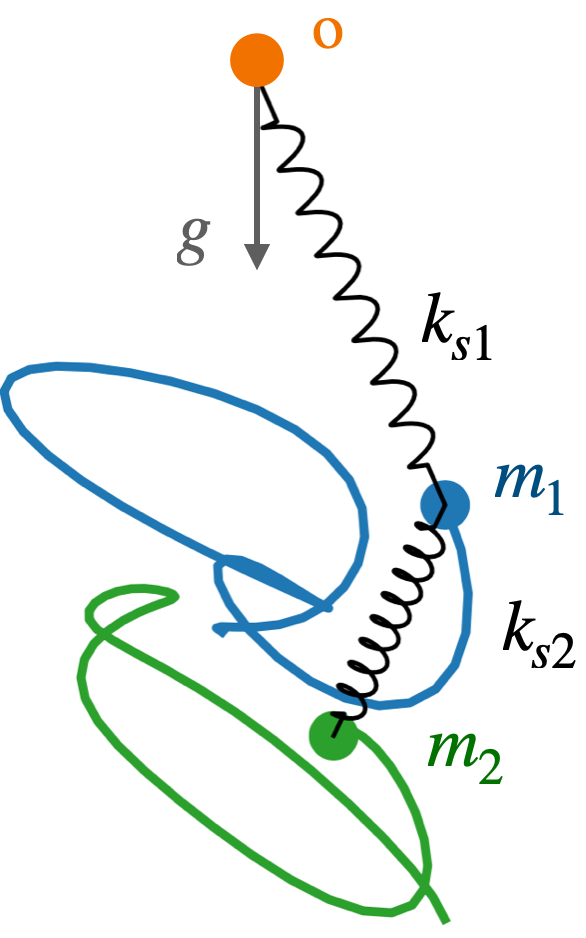}
    \caption{The springy double pendulum.}
    \label{fig:dp_system}
\end{figure}

In our experiments, the training set consists of positions and momenta of the pendula in a sequence of $\tilde T=10$ equispaced consecutive times sampled from a sequence of $T=60$ equispaced times obtained by integrating the dynamical system according to its ground truth parameters. 
In the testing stage, the trajectory at $t=1,\ldots,T$ from different initial states are predicted given the initializations at $t=0$. We consider three different testing setups to compare the dimensionless scalar-based implementation with the dimensional baseline considered in \citet{yao}. This baseline is currently state-of-the-art on this problem. It embodies Hamiltonian and geometric symmetries and performs very well \citep{yao}.

The test data used in Experiment~1 is generated from the same distribution as the training dataset. The test data used in Experiment~2 consists of applying a transformation to the test data in Experiment~1, where each of the input parameters that include a power of $\kg$ in its units ($m_1$, $m_2$, $k_{s1}$, $k_{s2}$, $\mathbf{p}_1(0)$ and $\mathbf{p}_2(0)$) is scaled by a factor randomly generated from $\Unif(3,7)$. The test data used in Experiment~3 has the input parameters $m_1$, $m_2$, $k_{s1}$, $k_{s2}$, $L_1$ and $L_2$ generated from $\Unif(1,5)$. 
We use the same training data $N=30000$ for all three experiments and each test set consists of $500$ data points. That is, Experiments~2 and 3 have out-of-distribution test data, relative to their training data.

We implement Hamiltonian neural networks (HNNs; \citealt{greydanus2019hnn,sanchezgonzalez2019hamiltonian}) with scalar-based MLPs for this learning task. In particular, we have a set of scalar inputs $\mathcal{S}=\{m_1,m_2,k_{s1},k_{s2},L_1,L_2\}$ and a set of vector inputs $\mathcal{V} = \{\mathbf{g}, \mathbf{p}_1(0),\mathbf{p}_2(0),\mathbf{q}_1(0) ,\mathbf{q}_2(0)-\mathbf{q}_1(0)\}$. We construct the dimensional scalars (baseline) and dimensionless scalars based on these two sets of inputs. 

The dimensional scalar inputs to the baseline MLPs include 32 scalars: 
(i) scalar inputs $\mathcal{S}$, as well as their inverses $\{1/a:a\in\mathcal{S}\}$;
(ii) inner products of the vector inputs $\{\mathbf{u}^\top\mathbf{v}:\mathbf{u},\mathbf{v}\in\mathcal{V}\}$, as well as their magnitudes $\{|\mathbf{u}|:\mathbf{u}\in\mathcal{V}\}$.  

The dimensionless scalar inputs are the following 32 scalars: (i) $m_1/m_2$, $k_{s1}/k_{s2}$, $L_1/L_2$ and their inverses;
(ii) we divide each vector input by its magnitude before we compute the inner products, which gives a set of dimensionless scalars $\{(\mathbf{u}^\top \mathbf{v})/(|\mathbf{u}||\mathbf{v}|):\mathbf{u},\mathbf{v}\in\mathcal{V}\}$;
(iii) we also consider dimensionless rational scalar monomials $(m_i\,|\mathbf{g}|)/(k_{si}\,L_i)$, $(k_{s_i}\,L_i)/(m_i\,|\mathbf{g}|)$, $|\mathbf{q}_i(0)|/L_i$, $|\mathbf{q}_i(0)|^2/L_i^2$, $|\mathbf{p}_i(0)| /( \sqrt{m_i\,k_{si}}\,L_i)$, $|\mathbf{p}_i(0)|^2 /(m_i\,k_{si}\,L_i^2)$, $i=1,2$. 
We use dimensionless scalars as inputs to the MLPs, which makes the outputs of the MLPs also dimensionless. 
The decoder then scales the outputs to restore the hamiltonian $H$ units of $\kg\,\m^2\,\s^{-2}$. At this stage we employ the following 26 scaling factors: $k_{sr}\,L_i\,L_j$, $m_i\,L_j\,|\mathbf{g}|$, $m_i\,\mathbf{g}^\top \mathbf{q}_j(0)$, $(\mathbf{p}_i(0)^\top \mathbf{p}_j(0))/m_r$, $k_{sr}\,\mathbf{q}_i(0)^\top \mathbf{q}_j(0)$, $i,j,r\in\{1,2\}$, all of which have the units of $\kg\,\m^2\,\s^{-2}$. 

The dimensional scalars-based and the dimensionless scalars-based MLPs both have equal numbers of model parameters, and are trained with the same set of hyper-parameters (number of training epochs, learning rate, etc.).

\begin{figure}[tp]
    \centering
    \includegraphics[width=0.85\textwidth]{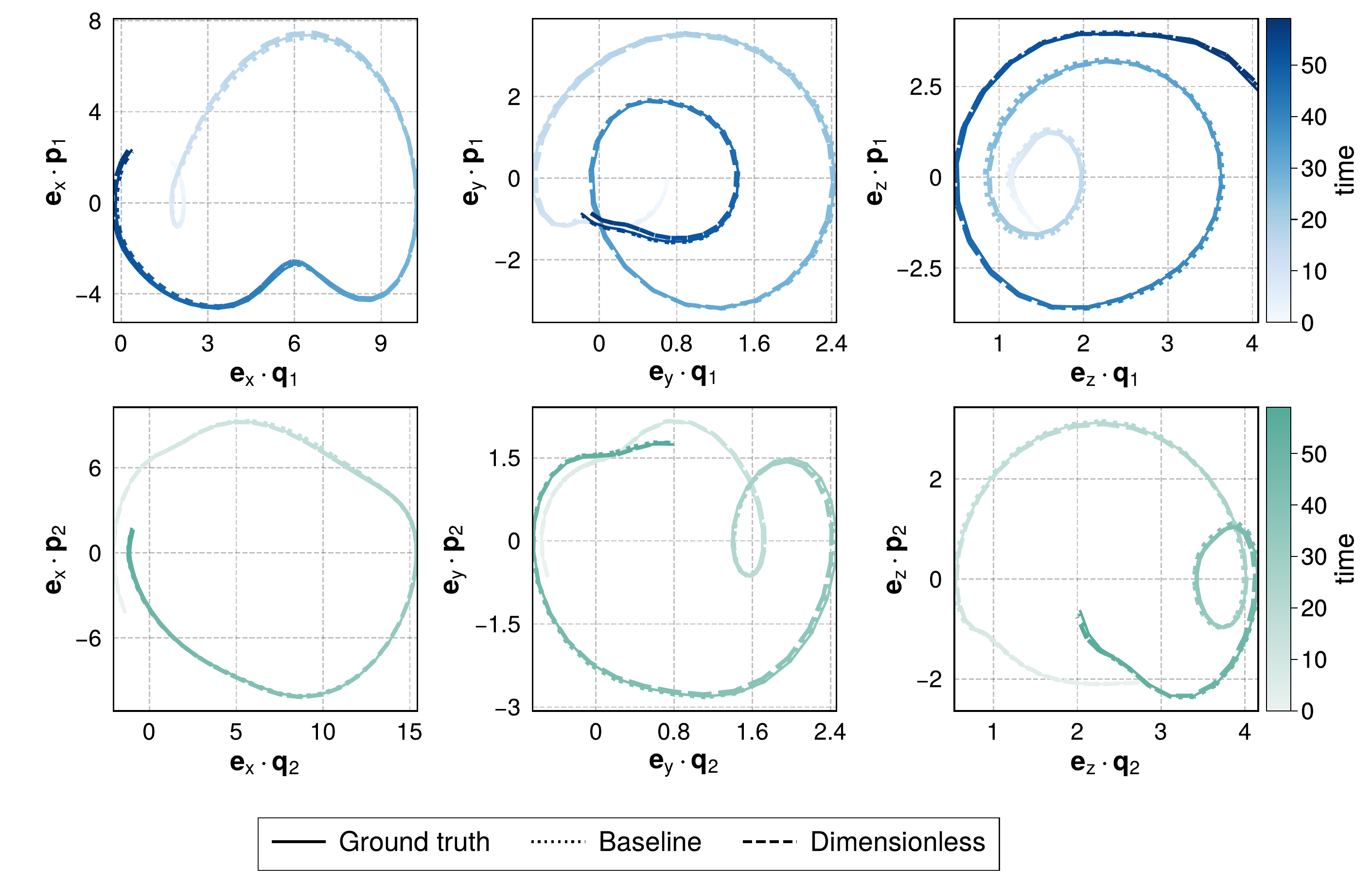}\\[3ex]
    \includegraphics[width=0.85\textwidth]{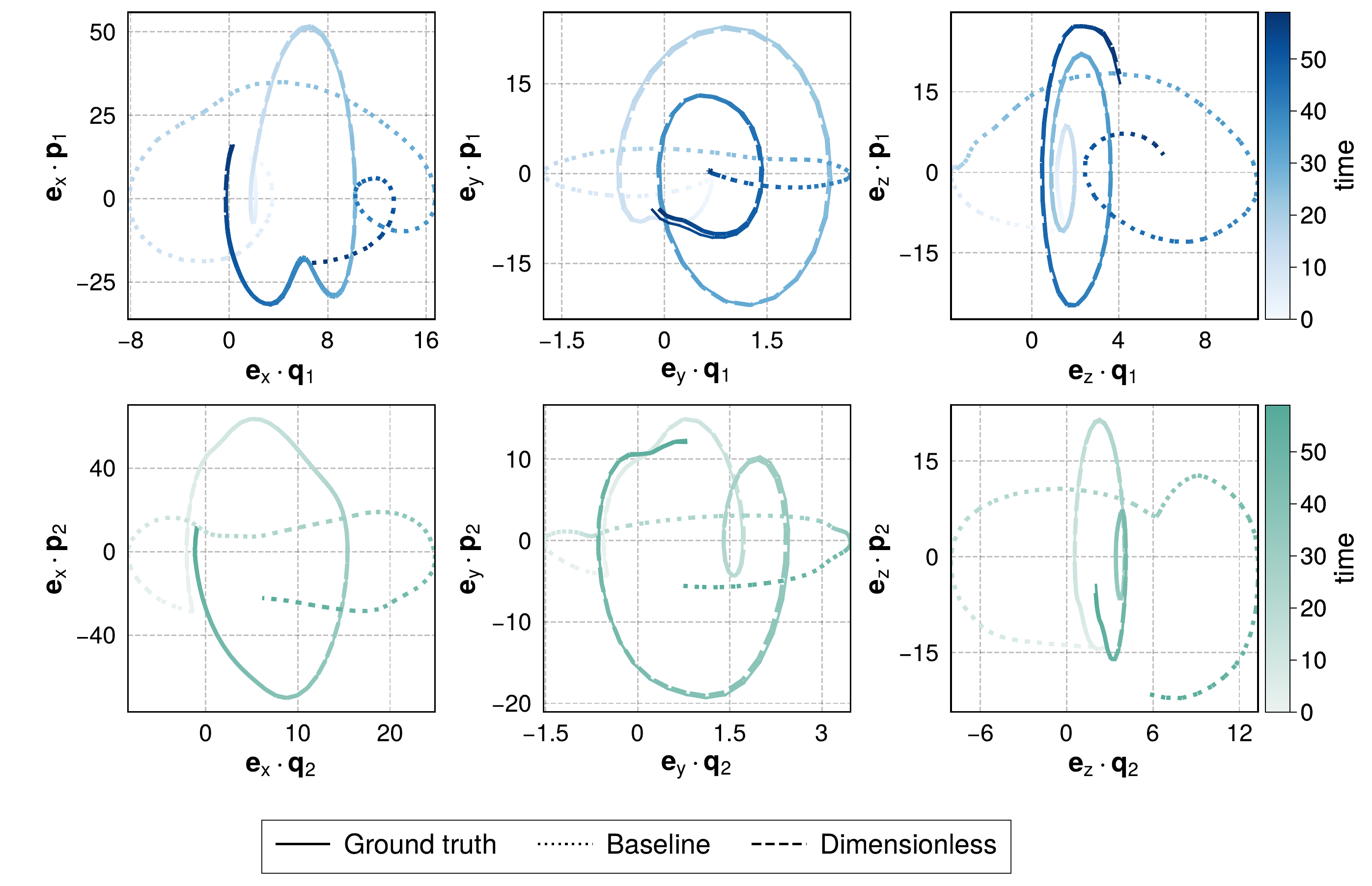} 
    \caption{Ground truth and predictions of mass 1 (top) and 2 (bottom) in the phase space w.r.t. each dimension. \textbf{Top 6 panels}: Results from Experiment~1, where the test data are generated from the same distribution as those used for training. Here the dimensional scalar based MLPs exhibit slightly more accurate predictions for longer time scales. \textbf{Bottom 6 panels}: Results from Experiment~2, where we use the same test data in Experiment~1 but each with its inputs that have units of $\kg$ randomly scaled by a factor generated from $\Unif(3,7)$. Here the dimensionless scalar based MLP is able to provide comparable performance to Experiment~1, while using the dimensional scalars gives much worse predictions.}
    \label{fig:double_pendulum}
\end{figure}
\begin{table}[t!]
    \centering
    \begin{tabular}{L{2.25cm} C{2.5cm} C{2.5cm} C{2.5cm} }
         \toprule
         Scalar-based MLPs   & Experiment~1 & Experiment~2 & Experiment~3\\
         \midrule
         Baseline   & $.0055\pm .0030$    & $.3669\pm .0050$  & $.1885\pm	.0031$\\
         \textbf{Dimensionless} & $.0061\pm .0024$ & $.0089\pm .0034$ & $.0435 \pm .0047$\\
         \bottomrule
    \end{tabular}
    \caption{Geometric mean (standard deviation computed over 10 trials) of state relative errors of the springy pendulum over $T=60$. Results are shown for the dimensional vs dimensionless scalar-based Hamiltonian Neural Networks (implemented as an MLP) on three different test sets. Test data used in Experiment~1 are generated from the same distribution as the training dataset; test data used in Experiment~2 using the same test data in Experiment~1 but each with its inputs that have units of $\kg$ randomly scaled by a factor generated from $\Unif(3,7)$; test data used in Experiment~3 has mass $m$, scalar spring constant $k_s$ and natural spring length $L$ generated from a different distribution.}
    \label{tbl:double_pendulumn}
\end{table} 

The prediction error (or state relative error) at time $t$ is defined as
\begin{align}
    \text{State.RelErr}(t) =  \frac{\sqrt{(\hat{\mathbf{z}}(t)-\mathbf{z}(t))^\top (\hat{\mathbf{z}}(t)-\mathbf{z}(t))}}{\sqrt{\hat{\mathbf z}(t)^\top\hat{\mathbf z}(t)}+\sqrt{\mathbf z(t)^\top \mathbf z(t)}}\label{eq:state_relerr}.
\end{align}
\tabref{tbl:double_pendulumn} reports the average errors over $\{t_1, \ldots, t_{60}\}$. When the test data are generated from the same distribution as the training data, the dimensional scalar based MLP exhibits slightly more accurate predictions for longer time scales using the same training hyper-parameters. When we have out-of-distribution test data as in Experiment~2 and 3, the performance of both methods deteriorate as expected, but the dimensionless scalar based MLP exhibits a significantly better generalization performance. In particular, if we rescale the units as in Experiment~2, where all the quantities that have the units of $\kg$ are scaled by the same randomly generated factor, the dimensionless scalar based MLP is able to provide comparable  performance to results from Experiment~1. Actually, this could be considered to be an in-distribution test set in the space of dimensionless scalars (see \secref{sec:bias}), and thus the only reason why the error is different is because the state relative error \eqref{eq:state_relerr} is not dimensionless.
In other words, our experiments show that imposing units equivariance increases the generalization performance significantly, especially in out-of-distribution settings.

Figure~\ref{fig:double_pendulum} provides an illustration of the predicted orbits by the dimensional and the dimensionless methods in Experiment~1 and Experiment~2.

\paragraph{Emulation: Arid vegetation model}
We further explore unit equivariance informed learning inspired by a non-linear problem in ecology\footnote{See \url{https://dwh.gg/Rietkerk}.}. In semi-arid environments, banded vegetation is a characteristic feature of plant self-organization which is modulated by the quantity of water available \citep{dagbovie2014pattern}. Inverting emergent vegetation patterns as a function of environmental changes is a central problem in ecology. Towards this end, a popular approach is the Rietkerk model, a set of differential equations relating surface water $u$, water absorbed into the soil $w$, and vegetation density $v$ \citep{rietkerk}.
\begin{table}[t]
    \centering
    \begin{tabular}{ c |c| c |c }
          & description & default & units \\
         \hline
         $R$ & rainfall  & $0.375 $  & $\, \l \, \d^{-1}\, \m^{-2}$\\
         $\alpha$ & infiltration rate & 0.2 & $\d^{-1}$ \\
         $k_2$ & saturation const. & 5 & $\g \, \m^{-2}$ \\
         $W_0$ & water infiltration const. & 0.1 & --- \\
         $D_u$ & surface water diffusion & 100 & $\d^{-1}\, \m^{2}$ \\
         $g_m$ & water uptake & 0.05 & $\l\, \g^{-1}\, \d^{-1}$ \\
         $k_1$ & water uptake constant & 5 & $\l\, \m^{-2}$ \\
         $\delta_w$ & soil water loss & 0.2 & $ \d^{-1}$ \\
         $D_w$ & soil water diffusion & 0.1 & $  \d^{-1}\, \m^2$ \\
         $c$ & water to biomass & 20 & $  \l^{-1}\, \g$ \\
         $\delta_v$ & vegetation loss & 0.25 & $ \d^{-1}$ \\
         $D_v$ & vegetation diffusion & 0.1 & $  \d^{-1}\, \m^2$ \\[2ex]
         $T$ & total integration time & 200 & $  \d$
         \\
        $\delta t$ & integration time step & 0.005 & $  \d$
        \\
        $L$ & integration patch length & 200 & $  \m$
        \\
        $\delta l$ & spatial step size & 2 & $\m$ \\
        \hline
    \end{tabular}
\begin{tabular}{c}
Dimensionless features\\
\hline
$c \, \alpha^{-1}\,g_m$ \\
$ R^{-1} \alpha \, k_1 $\\
$R^{-1}c^{-1} \alpha\, k_2 $ \\
$\alpha^{-1} \delta_w $\\
$\alpha^{-1} \delta_v$ \\
$W_0$ \\
$\alpha^{-1} D_v \, L^{-2}$ \\
$\alpha^{-1}D_u \, L^{-2}$ \\
$ \alpha\, T$ \\
$ \alpha\, \delta t$ \\
$ \alpha^{-1}D_w \, L^{-2} $\\
$ L^{-1}\delta l$ \\
\hline
\end{tabular}
    \caption{(Right) Parameters in the Rietkerk model and their units \citep{rietkerk}. The bottom four parameters are parameters of the integration. (Left) Basis of dimensionless features found by our method.}
    \label{table.params}
\end{table}
These differential equations are
\begin{align}\label{eq:rietkerk}
    \frac{\dd u}{\dd t} &= R - \alpha\,\frac{v + k_2\,W_0}{v + k_2}\,u + D_u\,\nabla^2 u\nonumber\\
    \frac{\dd w}{\dd t} &= \alpha\,\frac{v + k_2\,W_0}{v + k_2}\,u - g_m\,\frac{v\,w}{k_1 + w} - \delta_w\,w + D_w\,\nabla^2 w\nonumber\\
    \frac{\dd v}{\dd t} &= c\,g_m\,\frac{v\,w}{k_1 + w} - \delta_v\,v + D_v\,\nabla^2 v,
\end{align}
where $u,w,v$ are all functions of both two-dimensional spatial coordinates and time $t$, and the $\nabla^2$ operator is the scalar second derivative operator (Laplacian) with respect to position.
In detail, $u$ denotes the surface water density (units of $\mm = \l\,\m^{-2}$), $v$ is the soil water content (same units as $u$), and $v$ is the vegetation density (units of $\g\,\m^{-2})$. Further, the time derivative operator has units of $\d^{-1}$, the Laplacian operator has units of $\m^{-2}$, and the units of the other quantities ($R$, $w_0$, $g_m$, and so on) can be inferred from the equations in \eqref{eq:rietkerk}.
Here the natural base units are $(\l, \g, \d, \m)$.
Note that as there is a conversion $1000\,\l=1\,\m^3$,  we could, in principle, reduce the base units by one. However, there is no direct communication between water volume and distance across the surface, so these units can be kept separate.
In general, the units equivariance is more powerful when there are more independent base units, which leads to substantial design decisions for the investigator. The Rietkerk model is determined by a set of dimensional parameters described in Table~\ref{table.params}, and the initial conditions $u_0, v_0, w_0$. We consider random initial conditions, and random choice of parameters, uniformly sampled between 0.5 and 1.5 times the default value. For each choice of parameters we use finite differences to estimate the derivatives and Laplacian, and integrate the Rietkerk model using Euler's method with time step $0.005\, \d$, in a $200\,\m \times 200\, \m$ grid, with $2\,\m$ pixel spacing.

We consider the task of predicting, from initial conditions, the average vegetation density after $200$ days (the empirical steady state solution of \eqref{eq:rietkerk} at default parameters). We produce a training set of 1000 initial configurations and a test set of 100 configurations. A significant portion of these simulations ended up on total vegetation death at finite time. We didn't consider these examples for the regression task -- extending our results to classification is a future direction.  
\begin{figure}[tp]
    \centering
    \includegraphics[height=0.21\textwidth]{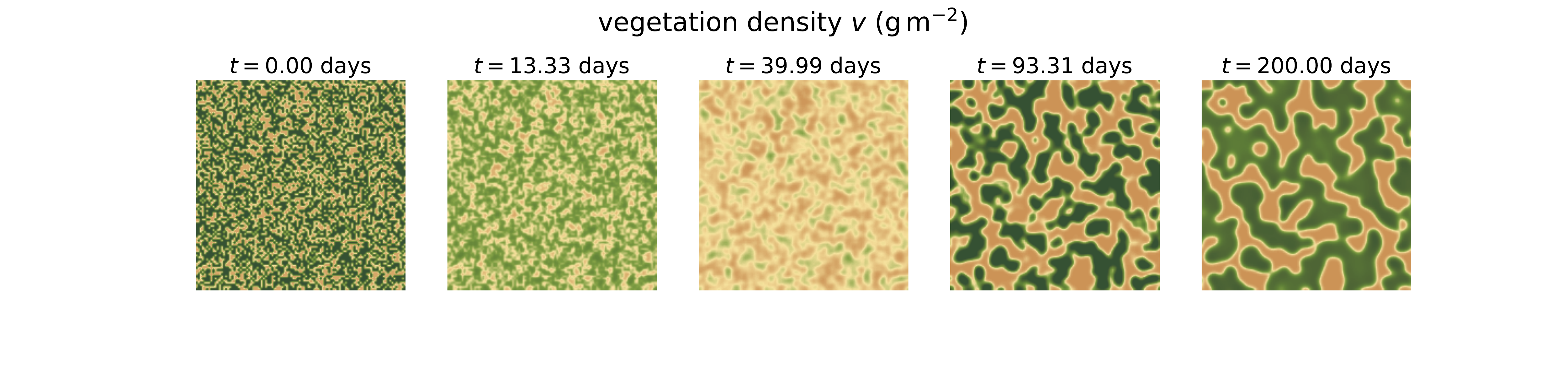}
    \includegraphics[height=0.17\textwidth]{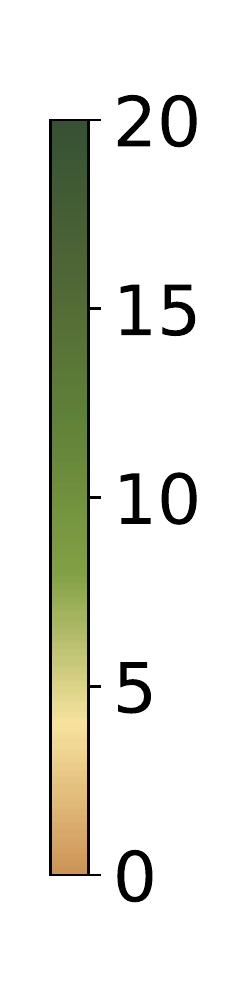}
    \\[4ex]
    \includegraphics[width=0.5\textwidth]{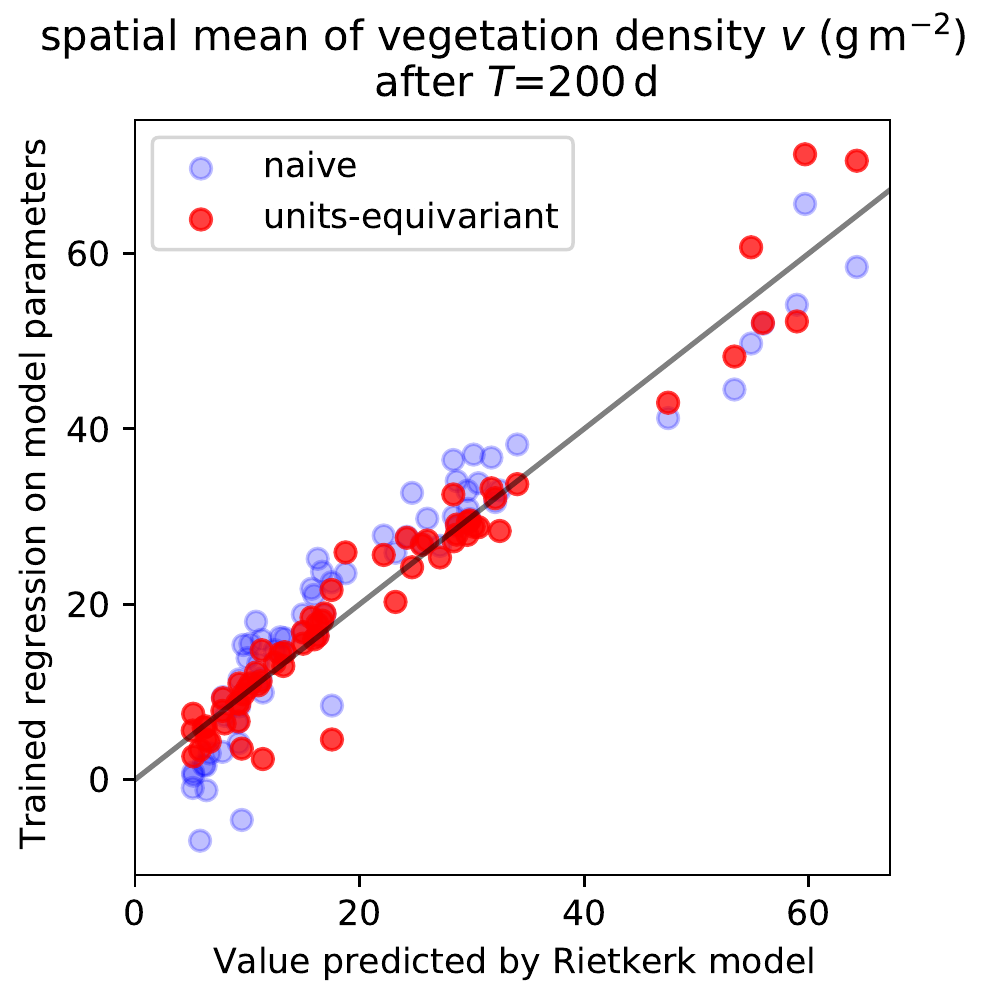}
    \caption{(Top) The evolution of vegetation density from random initialization according to a Rietkerk model. The model's parameters are given in Table~\ref{table.params}. (Bottom) We consider 1000 random initializations and evolutions according to random parameters sampled from a uniform distributions supported in $[0.5x , 1.5x]$ where $x$ are the baseline Rietkerk parameters from Table~\ref{table.params} (the integration parameters remain fixed). Our regression task is to predict the spatial mean vegetation after 200 days as a function of the Rietkerk parameters. The light blue dots show the (naive) regression vs Rietkerk model for a linear regression on the model parameters, its inverses and the constant 1 (33 features in total) on a held out test set. The dark red dots correspond to a linear regression using a basis of dimensionless features obtained with our (units-equivariant) method, their inverses and the constant 1 (25 features in total) in the same test set. The naive regression has a test MSE of  $26.3 \, \g^2 \, \m^{-4}$ whereas the units-equivariant regression has a MSE of $12.6 \, \g^2 \, \m^{-4}$. The Pearson correlation of the prediction and the target value is 0.94 for the naive regression and 0.97 for the units-equivariant regression. }
    \label{fig:vegetation}
\end{figure}
We perform two forms of linear regression, a baseline regression and a unitless regression. The baseline regression uses $33$ features: the dimensional parameters, their inverses, and the dimensionless constant 1 which describes affine linear functions. The dimensionless linear regression uses the method described in \secref{sec:approach}. It uses the Smith normal form to construct a basis of 12 dimensionless features, and it uses them, their inverses and the constant 1, obtaining 25 regression features. The results show that the dimensionless regression has significantly better performance in Figure \ref{fig:vegetation}.

Our toy model explores the impact of selecting dimensionally correct features on predicting average vegetation outcomes when data are generated from a well characterized ecological model. However, other interesting symbolic regression problems remain open.
For example, is the Rietkerk model considered the most appropriate for modelling banded vegetation patterns in general? A recent approach aimed to address the related inverse problem of determining the underlying structure of a nonlinear dynamical system from data \citep{brunton2016discovering}. There, sparse regression and compressed sensing tools informed the selection of a small number of informative, non-linear terms hypothesized to explain an observed dynamics. Thus, these kinds of problems present an intriguing venue for future exploration of units equivariance as a principled way to impose additional sparsity in a non-linear feature space which could further aid methods like that of \citet{brunton2016discovering} by restricting feature selection to units-equivariant, physically informative terms.

\paragraph{Symbolic regression: The black-body radiation law}
One of the most important moments in modern physics was the introduction of the quantum-mechanical constant $h$ by Planck around 1900 \citep{planck}.
In our language, this discovery can be seen as a symbolic regression, in which Planck discovered a simple symbolic expression that accurately summarized a host of data sets on radiating bodies at different temperatures.
The dimensional constant $h$ was introduced to explain the short-wavelength part of the radiation law, but it ended up being the governing constant for all quantum phenomena; it led to a simple prediction of the spectrum of the Hydrogen atom \citep{bohr} and is the core of the Schr\"odinger equation \citep{schrodinger}; it was extremely important in the history of physics.

The black-body radiation $B_\lambda(\lambda;T)$ from a perfectly radiating and absorbing (black) thermal source at temperature $T$ is properly measured in intensity units, which are (or can be) energy per time per wavelength per area per solid angle.
Because solid angles are dimensionless, this translates to SI units of $\J\,\m^{-3}\,\s^{-1}$.
The problem Planck faced was a set of measurements (labels) $B_\lambda(\lambda;T)$ at many wavelengths $\lambda$ for bodies at multiple temperatures $T$.
The input features $\lambda,T,c,\kB$ and output labels $B_\lambda(\lambda;T)$ of the problem are summarized in \tabref{tab:planck}, along with their units in the SI base unit system of $\kg,\m,\s,\K$.

\begin{table}[t]
    \centering
    \begin{tabular}{r|l|c|l}
    & \emph{description} & \emph{units} & \emph{comment} \\ \hline
    & & & \\[-1ex]
    $B_\lambda(\lambda;T)$ & intensity & $\kg\,\m^{-1}\,\s^{-3}$ & regression label \\
    & & & \\[-1ex]
    $\lambda$ & wavelength & $\m$ & variable feature\\
    $T$ & temperature & $\K$ & variable feature\\
    $c$ & speed of light & $\m\,\s^{-1}$ & fundamental constant\\
    $\kB$ & Boltzmann's constant & $\kg\,\m^{2}\,\s^{-2}\,\K^{-1}$ & fundamental constant\\
    \end{tabular}
    \caption{Labels and features in Planck's black-body radiation problem and their units.}
    \label{tab:planck}
\end{table}

In terms of the language illustrated in \figref{fig:approach}, the decoder $g_{\mathbf{x},\mathbf{v}}$ involves multiplying the dimensionless output of a dimensionless regression by a dimensional quantity with the same units as the labels $B_\lambda(\lambda;T)$.
The only possible dimensional quantity that can be made out of the features that matches the dimensions of the labels is
\begin{equation}
    \frac{c}{\lambda^4}\,\kB\,T ~,
\end{equation}
which has units of intensity.
The featurizer $\phi$ makes all possible dimensionless quantities out of the inputs.
But wait, there are no (non-trivial) dimensionless features possible!
In a dimensionless regression, the \emph{only} available input feature is the dimensionless constant \emph{unity}.
That is, in our approach, the only possible outcome of the regression in this case is
\begin{eqnarray}\label{eq:noh}
    B_\lambda(\lambda;T) = C\,\frac{c}{\lambda^4}\,\kB\,T ~,
\end{eqnarray}
where $C$ is a universal constant.
This is a classical dimensional-analysis result.

There are two comments to make here.
The first is that this form \eqref{eq:noh} is not a good fit to the data!
Thus the method we are proposing here fails.
The explanation for this failure is that there is a dimensionless constant, $h$ (now known as Planck's constant) that is missing from our formulation in \tabref{tab:planck}.
The second is that this form \eqref{eq:noh} is a perfect fit to the data \emph{at long wavelengths}.
That is, at long wavelengths, where quantum occupation numbers are high, the problem behaves classically, and the data are extremely well explained by \eqref{eq:noh}, with $C=2$.
This result (with $C=2$) is called the Rayleigh–-Jeans law.
If the reader is interested in the history of physics, the Rayleigh--Jeans law is the lynchpin of the ultraviolet catastrophe, which is a paradox of classical statistical mechanics, resolved by quantization.

What Planck discovered or realized is that the data could only be explained with the introduction of a new dimensional universal constant.
He had choices for the dimensions of this constant, but he set it to have dimensions of energy times time.
Planck's symbolic regression led to the complete expression
\begin{eqnarray}
    B_\lambda(\lambda;T) = \frac{2\,h\,c^2}{\lambda^5}\,\frac{1}{\exp(\frac{h\,c}{\lambda\,\kB\,T})-1}~,
\end{eqnarray}
which reduces to \eqref{eq:noh} with $C=2$ in the limit $\lambda\rightarrow\infty$.
This result required the introduction of the dimensional constant $h$, resolved the ultraviolet catastrophe, and seeded the discovery of quantum mechanics.
In the approach advocated in this work, we have no way to learn or discover missing dimensional constants.
That is a limitation of our approaches, and motivates future work.

\section{Discussion}

In the above, we defined units equivariance for machine learning, with a focus on regression and complex functions.
A function obeying this equivariance obeys the exact scalings that are required by the rules of dimensional analysis. These scalings must be obeyed by any theory or function in use in the natural sciences.

We developed a simple framework for implementing units equivariance into regression problems.
This framework puts burdens on the investigator---burdens of having consistent units for all inputs, and also a comprehensive list of dimensional constants---but is otherwise lightweight in terms of modifying existing regression methods.
We did not consider the important problems of learning dimensions, or discovery of missing dimensionless inputs, but these are worthy extensions of what we looked at here.

We argued that imposing units equivariance must improve the bias and variance of regression methods, both because it incorporates correct information, and also because it reduces model capacity at fixed complexity, often by an enormous factor.
The equivariance also enables out-of-sample generalization, because a test set that doesn't overlap a training set in dimensional inputs will often significantly overlap in dimensionless combinations of those inputs.
We illustrated these effects empirically with a few simple experiments.

Units equivariance applies to all functions in the natural sciences.
It won't be useful everywhere.
In particular, it is most useful when there are many independent units at play, and the full panoply of physical constants is known.
This is not true, say, for standard image-recognition tasks, for which all the inputs have the same units (intensity in image pixels) and the physical quantities (involved in the identification of pandas and kittens, say) are not known.
It is also not true in natural-science problems where there might be unknown physical constants or physical laws at play.
The discovery of physical laws is often the discovery of dimensional physical constants, as our black-body radiation law example problem (\secref{sec:experiments}) illustrates.

However, we are very optimistic about the usefulness of units equivariance in problems of emulation and symbolic regression.
In these settings, all symmetries are exact, and often all inputs (including all fundamental constants) are known (and have known units).
In particular, some of the cleanest physics problems might be in the area of the growth of structure in the Universe, where there are very few dimensioned quantities and the physics is dominated by one force (gravity).
These problems are of great interest at the present day, and have attracted very promising work with machine learning methods (for example, \citealt{he, berger, kodi, troster}).

\paragraph{Acknowledgments:}
It is a pleasure to thank
  Timothy Carson (Google),
  Miles Cranmer (Princeton),
  Samory Kpotufe (Columbia),
  Sanjoy Mahajan (Olin College),
  Bernhard Sch\"olkopf (MPI-IS),
  Kate Storey-Fisher (NYU), and
  Wenda Zhou (NYU and Flatiron Institute)
for valuable discussions.
We also thank the action editor Jean-Philippe Vert
and the anonymous reviewers for constructive feedback that helped us improve the manuscript significantly.
SV was partially supported by
ONR N00014-22-1-2126, 
the NSF–Simons Research Collaboration on the Mathematical and Scientific Foundations of Deep Learning
(MoDL) (NSF DMS 2031985),
NSF CISE 2212457,
an AI2AI Amazon research award, 
and the TRIPODS Institute for the Foundations of Graph and Deep Learning at Johns Hopkins University.

\bibliography{units}

\begin{thebibliography}{83}
\providecommand{\natexlab}[1]{#1}
\providecommand{\url}[1]{\texttt{#1}}
\expandafter\ifx\csname urlstyle\endcsname\relax
  \providecommand{\doi}[1]{doi: #1}\else
  \providecommand{\doi}{doi: \begingroup \urlstyle{rm}\Url}\fi

\bibitem[Arjovsky(2020)]{arjovsky2020out}
Martin Arjovsky.
\newblock \emph{Out of distribution generalization in machine learning}.
\newblock PhD thesis, New York University, 2020.

\bibitem[Baird(1992)]{baird1992document}
Henry~S Baird.
\newblock Document image defect models.
\newblock In \emph{Structured Document Image Analysis}, pages 546--556.
  Springer, 1992.

\bibitem[Bakarji et~al.(2022)Bakarji, Callaham, Brunton, and
  Kutz]{bakarji2022dimensionally}
Joseph Bakarji, Jared Callaham, Steven~L Brunton, and J~Nathan Kutz.
\newblock Dimensionally consistent learning with buckingham pi.
\newblock \emph{arXiv:2202.04643}, 2022.

\bibitem[Barenblatt(1996)]{barenblatt_1996}
Grigory~Isaakovich Barenblatt.
\newblock \emph{Scaling and transformation groups. Renormalization group}, page
  161–180.
\newblock Cambridge Texts in Applied Mathematics. Cambridge University Press,
  1996.
\newblock \doi{10.1017/CBO9781107050242.009}.

\bibitem[Batzner et~al.(2021)Batzner, Musaelian, Sun, Geiger, Mailoa,
  Kornbluth, Molinari, Smidt, and Kozinsky]{batzner2021se}
Simon Batzner, Albert Musaelian, Lixin Sun, Mario Geiger, Jonathan~P Mailoa,
  Mordechai Kornbluth, Nicola Molinari, Tess~E Smidt, and Boris Kozinsky.
\newblock Se (3)-equivariant graph neural networks for data-efficient and
  accurate interatomic potentials.
\newblock \emph{arXiv:2101.03164}, 2021.

\bibitem[Ben-David et~al.(2010)Ben-David, Blitzer, Crammer, Kulesza, Pereira,
  and Vaughan]{ben2010theory}
Shai Ben-David, John Blitzer, Koby Crammer, Alex Kulesza, Fernando Pereira, and
  Jennifer~Wortman Vaughan.
\newblock A theory of learning from different domains.
\newblock \emph{Machine learning}, 79\penalty0 (1):\penalty0 151--175, 2010.

\bibitem[Benton et~al.(2020)Benton, Finzi, Izmailov, and
  Wilson]{benton2020learning}
Gregory Benton, Marc Finzi, Pavel Izmailov, and Andrew~Gordon Wilson.
\newblock Learning invariances in neural networks.
\newblock \emph{arXiv:2010.11882}, 2020.

\bibitem[Berger and Stein(2019)]{berger}
Philippe Berger and George Stein.
\newblock A volumetric deep convolutional neural network for simulation of mock
  dark matter halo catalogues.
\newblock \emph{Monthly Notices of the Royal Astronomical Society},
  482\penalty0 (3):\penalty0 2861--2871, 2019.

\bibitem[Bietti et~al.(2021)Bietti, Venturi, and Bruna]{bietti2021sample}
Alberto Bietti, Luca Venturi, and Joan Bruna.
\newblock On the sample complexity of learning under geometric stability.
\newblock \emph{Advances in Neural Information Processing Systems}, 34, 2021.

\bibitem[Blum-Smith and Villar(2022)]{blum2022equivariant}
Ben Blum-Smith and Soledad Villar.
\newblock Equivariant maps from invariant functions.
\newblock \emph{arXiv preprint arXiv:2209.14991}, 2022.

\bibitem[Bohr(1913)]{bohr}
Niels Bohr.
\newblock I. on the constitution of atoms and molecules.
\newblock \emph{The London, Edinburgh, and Dublin Philosophical Magazine and
  Journal of Science}, 26\penalty0 (151):\penalty0 1--25, 1913.

\bibitem[Brugiapaglia et~al.(2021)Brugiapaglia, Liu, and
  Tupper]{brugiapaglia2021invariance}
Simone Brugiapaglia, M~Liu, and Paul Tupper.
\newblock Invariance, encodings, and generalization: learning identity effects
  with neural networks.
\newblock \emph{arXiv preprint arXiv:2101.08386}, 2021.

\bibitem[Brunton et~al.(2016)Brunton, Proctor, and
  Kutz]{brunton2016discovering}
Steven~L Brunton, Joshua~L Proctor, and J~Nathan Kutz.
\newblock Discovering governing equations from data by sparse identification of
  nonlinear dynamical systems.
\newblock \emph{Proceedings of the national academy of sciences}, 113\penalty0
  (15):\penalty0 3932--3937, 2016.

\bibitem[Buckingham(1914)]{buckingham1914pi}
Edgar Buckingham.
\newblock On physically similar systems; illustrations of the use of
  dimensional equations.
\newblock \emph{Physical Review}, 4\penalty0 (4):\penalty0 345--376, 1914.

\bibitem[Cahill et~al.(2020)Cahill, Mixon, and Parshall]{cahill2020lie}
Jameson Cahill, Dustin~G Mixon, and Hans Parshall.
\newblock Lie {PCA}: Density estimation for symmetric manifolds.
\newblock \emph{arXiv:2008.04278}, 2020.

\bibitem[{Cambioni} et~al.(2019){Cambioni}, {Asphaug}, {Emsenhuber}, {Gabriel},
  {Furfaro}, and {Schwartz}]{2019ApJ}
Saverio {Cambioni}, Erik {Asphaug}, Alexandre {Emsenhuber}, Travis S.~J.
  {Gabriel}, Roberto {Furfaro}, and Stephen~R. {Schwartz}.
\newblock Realistic on-the-fly outcomes of planetary collisions: Machine
  learning applied to simulations of giant impacts.
\newblock \emph{Astrophysical Journal}, 875\penalty0 (1):\penalty0 40, April
  2019.

\bibitem[{Cambioni} et~al.(2021){Cambioni}, {Jacobson}, {Emsenhuber},
  {Asphaug}, {Rubie}, {Gabriel}, {Schwartz}, and {Furfaro}]{2021PSJ}
Saverio {Cambioni}, Seth~A. {Jacobson}, Alexandre {Emsenhuber}, Erik {Asphaug},
  David~C. {Rubie}, Travis S.~J. {Gabriel}, Stephen~R. {Schwartz}, and Roberto
  {Furfaro}.
\newblock The effect of inefficient accretion on planetary differentiation.
\newblock \emph{Planetary Science Journal}, 2\penalty0 (3):\penalty0 93, June
  2021.

\bibitem[Chen and Villar(2022)]{chense}
Nan Chen and Soledad Villar.
\newblock {SE(3)}-equivariant self-attention via invariant features.
\newblock \emph{Machine Learning for Physics NeurIPS Workshop}, 2022.

\bibitem[Chen et~al.(2020{\natexlab{a}})Chen, Dobriban, and Lee]{chen2020group}
Shuxiao Chen, Edgar Dobriban, and Jane Lee.
\newblock A group-theoretic framework for data augmentation.
\newblock \emph{Advances in Neural Information Processing Systems},
  33:\penalty0 21321--21333, 2020{\natexlab{a}}.

\bibitem[Chen et~al.(2019{\natexlab{a}})Chen, Li, and Bruna]{chen2019cdsbm}
Zhengdao Chen, Lisha Li, and Joan Bruna.
\newblock Supervised community detection with line graph neural networks.
\newblock \emph{Internation Conference on Learning Representations},
  2019{\natexlab{a}}.

\bibitem[Chen et~al.(2019{\natexlab{b}})Chen, Villar, Chen, and
  Bruna]{chen2019equivalence}
Zhengdao Chen, Soledad Villar, Lei Chen, and Joan Bruna.
\newblock On the equivalence between graph isomorphism testing and function
  approximation with gnns.
\newblock In \emph{Advances in Neural Information Processing Systems}, pages
  15894--15902, 2019{\natexlab{b}}.

\bibitem[Chen et~al.(2020{\natexlab{b}})Chen, Chen, Villar, and
  Joan]{chen2020can}
Zhengdao Chen, Lei Chen, Soledad Villar, and Bruna Joan.
\newblock Can graph neural networks count substructures?
\newblock \emph{Advances in neural information processing systems},
  2020{\natexlab{b}}.

\bibitem[Cohen and Welling(2016)]{cohen2016group}
Taco~S. Cohen and Max Welling.
\newblock Group equivariant convolutional networks.
\newblock In \emph{Proceedings of the 33rd International Conference on
  International Conference on Machine Learning}, volume~48, page 2990–2999,
  2016.

\bibitem[Cohen and Welling(2017)]{Cohen2016steerable}
Taco~S Cohen and Max Welling.
\newblock Steerable cnns.
\newblock In \emph{International Conference on Learning Representations
  (ICLR)}, 2017.

\bibitem[Cohen et~al.(2018)Cohen, Geiger, Koehler, and
  Welling]{cohen2018spherical}
Taco~S. Cohen, Mario Geiger, Jonas Koehler, and Max Welling.
\newblock Spherical cnns, 2018.

\bibitem[Constantine et~al.(2017)Constantine, del Rosario, and
  Iaccarino]{constantine2017data}
Paul~G Constantine, Zachary del Rosario, and Gianluca Iaccarino.
\newblock Data-driven dimensional analysis: {Algorithms} for unique and
  relevant dimensionless groups.
\newblock \emph{arXiv:1708.04303}, 2017.

\bibitem[Cubuk et~al.(2018)Cubuk, Zoph, Mane, Vasudevan, and
  Le]{cubuk2018autoaugment}
Ekin~D Cubuk, Barret Zoph, Dandelion Mane, Vijay Vasudevan, and Quoc~V Le.
\newblock Autoaugment: Learning augmentation policies from data.
\newblock \emph{arXiv:1805.09501}, 2018.

\bibitem[Dagbovie and Sherratt(2014)]{dagbovie2014pattern}
Ayawoa~S Dagbovie and Jonathan~A Sherratt.
\newblock Pattern selection and hysteresis in the rietkerk model for banded
  vegetation in semi-arid environments.
\newblock \emph{Journal of The Royal Society Interface}, 11\penalty0
  (99):\penalty0 20140465, 2014.

\bibitem[Dao et~al.(2019)Dao, Gu, Ratner, Smith, De~Sa, and
  R{\'e}]{dao2019kernel}
Tri Dao, Albert Gu, Alexander Ratner, Virginia Smith, Chris De~Sa, and
  Christopher R{\'e}.
\newblock A kernel theory of modern data augmentation.
\newblock In \emph{International Conference on Machine Learning}, pages
  1528--1537. PMLR, 2019.

\bibitem[Dey et~al.(2022)Dey, De~Silva, LeVine, Shin, Xu, Geisa, Chu, Isik, and
  Vogelstein]{dey2022out}
Jayanta Dey, Ashwin De~Silva, Will LeVine, Jong Shin, Haoyin Xu, Ali Geisa,
  Tiffany Chu, Leyla Isik, and Joshua~T Vogelstein.
\newblock Out-of-distribution detection using kernel density polytopes.
\newblock \emph{arXiv:2201.13001}, 2022.

\bibitem[Duvenaud et~al.(2015)Duvenaud, Maclaurin, Iparraguirre, Bombarell,
  Hirzel, Aspuru-Guzik, and Adams]{duvenaud2015convolutional}
David~K Duvenaud, Dougal Maclaurin, Jorge Iparraguirre, Rafael Bombarell,
  Timothy Hirzel, Al{\'a}n Aspuru-Guzik, and Ryan~P Adams.
\newblock Convolutional networks on graphs for learning molecular fingerprints.
\newblock In \emph{Advances in neural information processing systems}, pages
  2224--2232, 2015.

\bibitem[Elesedy(2021)]{elesedy_kernel}
Bryn Elesedy.
\newblock Provably strict generalisation benefit for invariance in kernel
  methods.
\newblock \emph{arXiv:2106.02346}, 2021.

\bibitem[Elesedy and Zaidi(2021)]{elesedy2021provably}
Bryn Elesedy and Sheheryar Zaidi.
\newblock Provably strict generalisation benefit for equivariant models.
\newblock \emph{arXiv:2102.10333}, 2021.

\bibitem[Evangelou et~al.(2021)Evangelou, Wichrowski, Kevrekidis, Dietrich,
  Kooshkbaghi, McFann, and Kevrekidis]{evangelou2021parameter}
Nikolaos Evangelou, Noah~J Wichrowski, George~A Kevrekidis, Felix Dietrich,
  Mahdi Kooshkbaghi, Sarah McFann, and Ioannis~G Kevrekidis.
\newblock On the parameter combinations that matter and on those that do not.
\newblock \emph{arXiv:2110.06717}, 2021.

\bibitem[Finzi et~al.(2021)Finzi, Welling, and Wilson]{finzi}
Marc Finzi, Max Welling, and Andrew~Gordon Wilson.
\newblock A practical method for constructing equivariant multilayer
  perceptrons for arbitrary matrix groups.
\newblock \emph{arXiv:2104.09459}, 2021.

\bibitem[Frisone and Misiti(2019)]{frisone2019buckingham}
Federico Frisone and Andrea Misiti.
\newblock Buckingham theorem application to machine learning algorithms:
  methodology and practical examples.
\newblock Master's thesis, Politecnico di Milano, 2019.

\bibitem[Fuchs et~al.(2020)Fuchs, Worrall, Fischer, and Welling]{fuchs2020se}
Fabian Fuchs, Daniel Worrall, Volker Fischer, and Max Welling.
\newblock Se (3)-transformers: 3d roto-translation equivariant attention
  networks.
\newblock \emph{Advances in Neural Information Processing Systems}, 33, 2020.

\bibitem[Fulton and Harris(2013)]{fulton2013representation}
William Fulton and Joe Harris.
\newblock \emph{Representation theory: a first course}, volume 129.
\newblock Springer Science \& Business Media, 2013.

\bibitem[Gama et~al.(2020)Gama, Isufi, Leus, and Ribeiro]{gama2020graphs}
Fernando Gama, Elvin Isufi, Geert Leus, and Alejandro Ribeiro.
\newblock Graphs, convolutions, and neural networks: From graph filters to
  graph neural networks.
\newblock \emph{IEEE Signal Processing Magazine}, 37\penalty0 (6):\penalty0
  128--138, 2020.

\bibitem[Geisa et~al.(2021)Geisa, Mehta, Helm, Dey, Eaton, Dick, Priebe, and
  Vogelstein]{geisa2021towards}
Ali Geisa, Ronak Mehta, Hayden~S Helm, Jayanta Dey, Eric Eaton, Jeffery Dick,
  Carey~E Priebe, and Joshua~T Vogelstein.
\newblock Towards a theory of out-of-distribution learning.
\newblock \emph{arXiv:2109.14501}, 2021.

\bibitem[Gilmer et~al.(2017)Gilmer, Schoenholz, Riley, Vinyals, and
  Dahl]{gilmer2017neural}
Justin Gilmer, Samuel~S Schoenholz, Patrick~F Riley, Oriol Vinyals, and
  George~E Dahl.
\newblock Neural message passing for quantum chemistry.
\newblock In \emph{Proceedings of the 34th International Conference on Machine
  Learning-Volume 70}, pages 1263--1272. JMLR. org, 2017.

\bibitem[Greydanus et~al.(2019)Greydanus, Dzamba, and
  Yosinski]{greydanus2019hnn}
Samuel Greydanus, Misko Dzamba, and Jason Yosinski.
\newblock Hamiltonian neural networks.
\newblock In H.~Wallach, H.~Larochelle, A.~Beygelzimer, F.~d\textquotesingle
  Alch\'{e}-Buc, E.~Fox, and R.~Garnett, editors, \emph{Advances in Neural
  Information Processing Systems}, volume~32, 2019.

\bibitem[Gripaios et~al.(2021)Gripaios, Haddadin, and
  Lester]{gripaios2021lorentz}
Ben Gripaios, Ward Haddadin, and Christopher~G Lester.
\newblock Lorentz-and permutation-invariants of particles.
\newblock \emph{Journal of Physics A: Mathematical and Theoretical},
  54\penalty0 (15):\penalty0 155201, 2021.

\bibitem[Haddadin(2021)]{haddadin2021invariant}
Ward Haddadin.
\newblock Invariant polynomials and machine learning.
\newblock \emph{arXiv:2104.12733}, 2021.

\bibitem[Hanneke and Kpotufe(2019)]{hanneke2019value}
Steve Hanneke and Samory Kpotufe.
\newblock On the value of target data in transfer learning.
\newblock \emph{Advances in Neural Information Processing Systems}, 32, 2019.

\bibitem[He et~al.(2019)He, Li, Feng, Ho, Ravanbakhsh, Chen, and
  P{\'o}czos]{he}
Siyu He, Yin Li, Yu~Feng, Shirley Ho, Siamak Ravanbakhsh, Wei Chen, and
  Barnab{\'a}s P{\'o}czos.
\newblock Learning to predict the cosmological structure formation.
\newblock \emph{Proceedings of the National Academy of Sciences}, 116\penalty0
  (28):\penalty0 13825--13832, 2019.

\bibitem[Huang and Villar(2021)]{huang2021short}
Ningyuan~Teresa Huang and Soledad Villar.
\newblock A short tutorial on the weisfeiler-lehman test and its variants.
\newblock In \emph{ICASSP 2021-2021 IEEE International Conference on Acoustics,
  Speech and Signal Processing (ICASSP)}, pages 8533--8537. IEEE, 2021.

\bibitem[Hubert and Labahn(2012)]{hubert2012rational}
Evelyne Hubert and George Labahn.
\newblock Rational invariants of scalings from {H}ermite normal forms.
\newblock In \emph{Proceedings of the 37th International Symposium on Symbolic
  and Algebraic Computation}, pages 219--226, 2012.

\bibitem[Jin et~al.(2020)Jin, Barzilay, and Jaakkola]{jin2020composing}
Wengong Jin, Regina Barzilay, and Tommi Jaakkola.
\newblock Composing molecules with multiple property constraints.
\newblock \emph{arXiv:2002.03244}, 2020.

\bibitem[Kang and Feng(2018)]{kang2018transferable}
Bingyi Kang and Jiashi Feng.
\newblock Transferable meta learning across domains.
\newblock In \emph{UAI}, pages 177--187, 2018.

\bibitem[Kashinath et~al.(2021)Kashinath, Mustafa, Albert, Wu, Jiang,
  Esmaeilzadeh, Azizzadenesheli, Wang, Chattopadhyay, Singh,
  et~al.]{kashinath2021physics}
K~Kashinath, M~Mustafa, A~Albert, JL~Wu, C~Jiang, S~Esmaeilzadeh,
  K~Azizzadenesheli, R~Wang, A~Chattopadhyay, A~Singh, et~al.
\newblock Physics-informed machine learning: {Case} studies for weather and
  climate modelling.
\newblock \emph{Philosophical Transactions of the Royal Society A},
  379\penalty0 (2194):\penalty0 20200093, 2021.

\bibitem[Kodi~Ramanah et~al.(2020)Kodi~Ramanah, Charnock, Villaescusa-Navarro,
  and Wandelt]{kodi}
Doogesh Kodi~Ramanah, Tom Charnock, Francisco Villaescusa-Navarro, and
  Benjamin~D Wandelt.
\newblock Super-resolution emulator of cosmological simulations using deep
  physical models.
\newblock \emph{Monthly Notices of the Royal Astronomical Society},
  495\penalty0 (4):\penalty0 4227--4236, 2020.

\bibitem[Kondor(2018)]{kondor2018n}
Risi Kondor.
\newblock N-body networks: a covariant hierarchical neural network architecture
  for learning atomic potentials.
\newblock \emph{arXiv:1803.01588}, 2018.

\bibitem[LeCun et~al.(1989)LeCun, Boser, Denker, Henderson, Howard, Hubbard,
  and Jackel]{lecun1989backpropagation}
Yann LeCun, Bernhard Boser, John~S Denker, Donnie Henderson, Richard~E Howard,
  Wayne Hubbard, and Lawrence~D Jackel.
\newblock Backpropagation applied to handwritten zip code recognition.
\newblock \emph{Neural computation}, 1\penalty0 (4):\penalty0 541--551, 1989.

\bibitem[Li et~al.(2018)Li, Yang, Song, and Hospedales]{li2018learning}
Da~Li, Yongxin Yang, Yi-Zhe Song, and Timothy~M Hospedales.
\newblock Learning to generalize: Meta-learning for domain generalization.
\newblock In \emph{Thirty-Second AAAI Conference on Artificial Intelligence},
  2018.

\bibitem[Liang and Zou(2022)]{liang2022metashift}
Weixin Liang and James Zou.
\newblock Metashift: A dataset of datasets for evaluating contextual
  distribution shifts and training conflicts.
\newblock \emph{arXiv:2202.06523}, 2022.

\bibitem[Mansour et~al.(2009)Mansour, Mohri, and
  Rostamizadeh]{mansour2009domain}
Yishay Mansour, Mehryar Mohri, and Afshin Rostamizadeh.
\newblock Domain adaptation: Learning bounds and algorithms.
\newblock \emph{arXiv:0902.3430}, 2009.

\bibitem[Maron et~al.(2018)Maron, Ben-Hamu, Shamir, and
  Lipman]{maron2018invariant}
Haggai Maron, Heli Ben-Hamu, Nadav Shamir, and Yaron Lipman.
\newblock Invariant and equivariant graph networks.
\newblock In \emph{International Conference on Learning Representations}, 2018.

\bibitem[Mei et~al.(2021)Mei, Misiakiewicz, and Montanari]{mei2021learning}
Song Mei, Theodor Misiakiewicz, and Andrea Montanari.
\newblock Learning with invariances in random features and kernel models, 2021.

\bibitem[Morris et~al.(2019)Morris, Ritzert, Fey, Hamilton, Lenssen, Rattan,
  and Grohe]{morris2019higher}
Christopher Morris, Martin Ritzert, Matthias Fey, William~L Hamilton, Jan~Eric
  Lenssen, Gaurav Rattan, and Martin Grohe.
\newblock Weisfeiler and leman go neural: Higher-order graph neural networks.
\newblock \emph{Association for the Advancement of Artificial Intelligence},
  2019.

\bibitem[Planck(1901)]{planck}
Max Planck.
\newblock On the law of the energy distribution in the normal spectrum.
\newblock \emph{Ann. Phys}, 4\penalty0 (553):\penalty0 1--11, 1901.

\bibitem[Portilheiro(2022)]{portilheiro2022tradeoff}
Vasco Portilheiro.
\newblock A tradeoff between universality of equivariant models and
  learnability of symmetries.
\newblock \emph{arXiv preprint arXiv:2210.09444}, 2022.

\bibitem[Rietkerk et~al.(2002)Rietkerk, Boerlijst, van Langevelde,
  HilleRisLambers, de~Koppel, Kumar, Prins, and de~Roos]{rietkerk}
Max Rietkerk, Maarten~C Boerlijst, Frank van Langevelde, Reinier
  HilleRisLambers, Johan~van de~Koppel, Lalit Kumar, Herbert~HT Prins, and
  Andr{\'e}~M de~Roos.
\newblock Self-organization of vegetation in arid ecosystems.
\newblock \emph{The American Naturalist}, 160\penalty0 (4):\penalty0 524--530,
  2002.

\bibitem[Rovelli and Gaul(2000)]{rovelli2000loop}
Carlo Rovelli and Marcus Gaul.
\newblock Loop quantum gravity and the meaning of diffeomorphism invariance.
\newblock In \emph{Towards quantum gravity}, pages 277--324. Springer, 2000.

\bibitem[Rudolph et~al.(1998)]{rudolph1998context}
Stephan Rudolph et~al.
\newblock On the context of dimensional analysis in artificial intelligence.
\newblock In \emph{International Workshop on Similarity Methods}. Citeseer,
  1998.

\bibitem[Sanchez-Gonzalez et~al.(2019)Sanchez-Gonzalez, Bapst, Cranmer, and
  Battaglia]{sanchezgonzalez2019hamiltonian}
Alvaro Sanchez-Gonzalez, Victor Bapst, Kyle Cranmer, and Peter Battaglia.
\newblock Hamiltonian graph networks with {ODE} integrators, 2019.

\bibitem[Schr{\"o}dinger(1926)]{schrodinger}
Erwin Schr{\"o}dinger.
\newblock An undulatory theory of the mechanics of atoms and molecules.
\newblock \emph{Physical review}, 28\penalty0 (6):\penalty0 1049, 1926.

\bibitem[Shafarevich(1994)]{shafarevich}
Igor~R Shafarevich.
\newblock \emph{Basic Algebraic Geometry 1}.
\newblock Springer-Verlag Berlin/Heidelberg, second edition, 1994.

\bibitem[Shen et~al.(2022)Shen, Bubeck, and Gunasekar]{shen2022data}
Ruoqi Shen, S{\'e}bastien Bubeck, and Suriya Gunasekar.
\newblock Data augmentation as feature manipulation: {A} story of desert cows
  and grass cows.
\newblock \emph{arXiv:2203.01572}, 2022.

\bibitem[Stanley(2016)]{stanley2016smithnormalform}
Richard~P. Stanley.
\newblock Smith normal form in combinatorics.
\newblock \emph{Journal of Combinatorial Theory, Series A}, 144:\penalty0
  476--495, 2016.

\bibitem[Thomas et~al.(2018)Thomas, Smidt, Kearnes, Yang, Li, Kohlhoff, and
  Riley]{thomas2018tensor}
Nathaniel Thomas, Tess Smidt, Steven Kearnes, Lusann Yang, Li~Li, Kai Kohlhoff,
  and Patrick Riley.
\newblock Tensor field networks: {Rotation-} and translation-equivariant neural
  networks for 3d point clouds.
\newblock \emph{arXiv:1802.08219}, 2018.

\bibitem[Thompson and Taylor(2008)]{si}
Ambler Thompson and Barry~N. Taylor.
\newblock \emph{Guide for the Use of the International System of Units (SI);
  Natl. Inst. Stand. Technol. Spec. Publ. 811, 2008 ed}.
\newblock National Institute of Standards and Technology, 2008.

\bibitem[Tr{\"o}ster et~al.(2019)Tr{\"o}ster, Ferguson, Harnois-D{\'e}raps, and
  McCarthy]{troster}
Tilman Tr{\"o}ster, Cameron Ferguson, Joachim Harnois-D{\'e}raps, and Ian~G
  McCarthy.
\newblock Painting with baryons: {Augmenting} {N}-body simulations with gas
  using deep generative models.
\newblock \emph{Monthly Notices of the Royal Astronomical Society: Letters},
  487\penalty0 (1):\penalty0 L24--L29, 2019.

\bibitem[Van~Dyk and Meng(2001)]{van2001art}
David~A Van~Dyk and Xiao-Li Meng.
\newblock The art of data augmentation.
\newblock \emph{Journal of Computational and Graphical Statistics}, 10\penalty0
  (1):\penalty0 1--50, 2001.

\bibitem[Villar et~al.(2021)Villar, Hogg, Storey-Fisher, Yao, and
  Blum-Smith]{villar}
Soledad Villar, David~W Hogg, Kate Storey-Fisher, Weichi Yao, and Ben
  Blum-Smith.
\newblock Scalars are universal: Equivariant machine learning, structured like
  classical physics.
\newblock In \emph{Thirty-Fifth Conference on Neural Information Processing
  Systems}, 2021.

\bibitem[Wang et~al.(2020{\natexlab{a}})Wang, Jia, Cao, and
  Gong]{wang2020certified}
Binghui Wang, Jinyuan Jia, Xiaoyu Cao, and Neil~Zhenqiang Gong.
\newblock Certified robustness of graph neural networks against adversarial
  structural perturbation.
\newblock \emph{arXiv:2008.10715}, 2020{\natexlab{a}}.

\bibitem[Wang et~al.(2020{\natexlab{b}})Wang, Kashinath, Mustafa, Albert, and
  Yu]{wang2020towards}
Rui Wang, Karthik Kashinath, Mustafa Mustafa, Adrian Albert, and Rose Yu.
\newblock Towards physics-informed deep learning for turbulent flow prediction.
\newblock In \emph{Proceedings of the 26th ACM SIGKDD International Conference
  on Knowledge Discovery \& Data Mining}, pages 1457--1466, 2020{\natexlab{b}}.

\bibitem[Weiler and Cesa(2019)]{Weiler2019e2equivariant}
Maurice Weiler and Gabriele Cesa.
\newblock General e(2)-equivariant steerable cnns.
\newblock In H.~Wallach, H.~Larochelle, A.~Beygelzimer, F.~d\textquotesingle
  Alch\'{e}-Buc, E.~Fox, and R.~Garnett, editors, \emph{Advances in Neural
  Information Processing Systems}, volume~32, 2019.

\bibitem[Weiler et~al.(2018)Weiler, Geiger, Welling, Boomsma, and
  Cohen]{weiler20183d}
Maurice Weiler, Mario Geiger, Max Welling, Wouter Boomsma, and Taco Cohen.
\newblock 3d steerable cnns: Learning rotationally equivariant features in
  volumetric data, 2018.

\bibitem[Wong et~al.(2016)Wong, Gatt, Stamatescu, and
  McDonnell]{wong2016understanding}
Sebastien~C Wong, Adam Gatt, Victor Stamatescu, and Mark~D McDonnell.
\newblock Understanding data augmentation for classification: when to warp?
\newblock In \emph{2016 international conference on digital image computing:
  techniques and applications (DICTA)}, pages 1--6. IEEE, 2016.

\bibitem[Xu et~al.(2018)Xu, Hu, Leskovec, and Jegelka]{xu2018powerful}
Keyulu Xu, Weihua Hu, Jure Leskovec, and Stefanie Jegelka.
\newblock How powerful are graph neural networks?
\newblock \emph{arXiv:1810.00826}, 2018.

\bibitem[Yao et~al.(2021)Yao, Storey-Fisher, Hogg, and Villar]{yao}
Weichi Yao, Kate Storey-Fisher, David~W Hogg, and Soledad Villar.
\newblock A simple equivariant machine learning method for dynamics based on
  scalars.
\newblock \emph{arXiv:2110.03761}, 2021.

\bibitem[Yu et~al.(2021)Yu, Perdikaris, and Karpatne]{yu-physics}
Rose Yu, Paris Perdikaris, and Anuj Karpatne.
\newblock Physics-guided ai for large-scale spatiotemporal data.
\newblock In \emph{Proceedings of the 27th ACM SIGKDD Conference on Knowledge
  Discovery \& Data Mining}, KDD '21, page 4088–4089, New York, NY, USA,
  2021. Association for Computing Machinery.
\newblock ISBN 9781450383325.
\newblock \doi{10.1145/3447548.3470793}.
\newblock URL \url{https://doi.org/10.1145/3447548.3470793}.

\end{thebibliography}

\clearpage\appendix

\section{Generalization improvements}\label{sec:analysis}

We first explain ideas from \citet{elesedy2021provably} developed in the context of invariant and equivariant functions with respect to actions by compact groups.

Given $X\subseteq \mathbb R^d$ and $G$ a compact group acting on $X$, we fix a $G$-invariant measure $\mu$. We consider the Hilbert space of functions $\mathcal H=\mathscr L_2(X, \mu)$. The action of $G$ on $X$ induces a natural action on $\mathcal H$ via the formula
\[
[\lambda\cdot f](x) := f(\lambda^{-1}\cdot x)
\]
for $x\in X$ and $\lambda \in G$. We split $\mathcal H$ into two orthogonal components, the closed \emph{ground-truth} $G$-invariant subspace $\bar{ \mathcal H}$ consisting of the functions $f\in\mathcal H$ satisfying $\lambda\cdot f = f$ for all $\lambda\in G$, and its orthogonal complement $\mathcal H^\perp$, so that $\mathcal H = \bar{ \mathcal H }\oplus \mathcal H^\perp$. Using that $\mu$ is $G$-invariant, it is noted \citep{elesedy2021provably} that the orthogonal projection onto $\bar{ \mathcal H}$ coincides with averaging along the $G$-orbit, also known as the Reynolds operator:
\begin{equation} \label{eq.projection}
    \mathcal O f(x) = \int_{G} f(\lambda \cdot x) d\lambda,
\end{equation}
where $\lambda$ is the normalized Haar measure of the group. The proof from \cite{elesedy2021provably} that orthogonal projection to $\bar{\mathcal H}$ coincides with $\mathcal O$ is reproduced below in \eqref{eq.sa1}-\eqref{eq.sa4}, and an alternative proof follows as well.

The Reynolds operator has good algebraic properties. It can be alternatively characterized as the unique $G$-invariant projection $\mathcal H \rightarrow \bar {\mathcal H}$, i.e., the unique linear map $\mathcal O: \mathcal H \rightarrow \bar {\mathcal H}$ that restricts to the identity on $\bar {\mathcal H}$ and satisfies 
\begin{equation}
\mathcal O (\lambda \cdot f) = \mathcal O f
\end{equation}
for $\lambda \in G$ and $f\in \mathcal H$. Yet another characterization is that it restricts to the identity on any subspace consisting of invariants, and to zero on any $G$-stable closed subspace not containing nontrivial invariants. The Hilbert space $\mathcal H$ contains the algebra of compactly-supported continuous functions $C_c(X)$ as a dense subspace, and the Reynolds operator has the property that its restriction to this algebra commutes with multiplication by invariant such functions, i.e., it satisfies
\begin{equation}
    \mathcal O (fh) = f\mathcal O h
\end{equation}
for all $f\in C_c(X)^G$ (the subalgebra of invariant compactly-supported continuous functions) and $h\in C_c(X)$. In other words, $\mathcal O$ restricts to a $C_c(X)^G$-module projection $C_c(X)\rightarrow C_c(X)^G$.

Now we verify the observation from \cite{elesedy2021provably} that $\mathcal O$ is the orthogonal projection in $\mathcal H$ onto the space of invariants; equivalently, given $f\in \mathcal H$, we have
$\arg\min_{h \in \bar {\mathcal H}} \| h - f\|_{\mu}^2 =\mathcal O(f)$. 
In order to show this, it suffices to observe that $\mathcal 
O$ is self-adjoint with respect to the inner product in $\mathscr L_2(X, \mu)$:
\begin{eqnarray}
 \langle \mathcal O f, h \rangle_\mu &=& \int_X \left\langle \int_G f(\lambda \cdot x) d\lambda , h(x) \right \rangle d\mu(x) \label{eq.sa1}\\
&=& \int_X \int_G \langle  f(\lambda \cdot x) , h(x) \rangle  d\lambda \, d\mu(x) \\
&=& \int_X \int_G \langle  f( x) , h(\lambda^{-1}\cdot x) \rangle  d\lambda \, d\mu(x) \label{eq.ginverse}\\
&=& \langle f,  \mathcal O h \rangle_\mu. \label{eq.sa4}
\end{eqnarray}
Note that \eqref{eq.ginverse} holds due to $\mu$ being $G$-invariant. An alternative, more conceptual way to see that $\mathcal O$ is orthogonal projection to $\bar{\mathcal H}$ under the assumption that $\mu$ is $G$-invariant is as follows. Because $\mu$ is $G$-invariant, the inner product on $\mathcal H = \mathscr L_2(X,\mu)$ is also $G$-invariant. Thus $G$ acts by unitary transformations on $\mathcal H$. This implies that every $G$-invariant element of $\mathcal H$ is orthogonal to every $G$-stable closed subspace not containing any nontrivial $G$-invariants. This is because if $f\in \mathcal H$ is invariant and $h\in V\subset\mathcal{H}$ where $V$ is a closed, $G$-stable subspace containing no nontrivial invariants, then $\langle f,h\rangle = \int_G\langle \lambda\cdot f, \lambda \cdot h\rangle d\lambda = \langle f, \mathcal O h\rangle = \langle f,0\rangle =0$. The integral is with respect to normalized Haar measure on $G$. The first equality is the fact that $\lambda$ is unitary, the second is because $f$ is invariant so the integral can be pushed into the inner product, and the third is because $\mathcal O h$ is both invariant and in $V$, therefore trivial. As discussed above, $\mathcal O$ acts as the identity on invariants and annihilates $G$-stable closed subspaces containing no nontrivial invariants (and this also follows from the calculation just performed). Since the former are orthogonal to the latter, this means $\mathcal{O}$ is the orthogonal projection onto the subspace of invariants in $\mathcal H$.

Now, given $f^*:X \to \mathbb R^k$, 
$f^*\in \bar{ \mathcal H}$, \citet{elesedy2021provably} consider data $y=f^*(x)+\xi$ where $\xi$ is sampled from a zero-mean, finite variance distribution in $\mathbb R^k$. 
The risk of a function $f$ is the expected value of the prediction error
\begin{equation} \label{eq.risk}
     \mathcal R(f) = \mathbb E_{x\sim \mu} \|y - f(x)\|_2^2
\end{equation}
and given two functions $f$ and $f'$ the generalization gap is
\begin{equation}
    \Delta(f, f')= \mathcal R(f)- \mathcal R(f'). 
\end{equation}
In \citet{elesedy2021provably} a regression problem is considered in which the regression is performed on a subspace $U \subset \mathcal H$ that is closed under \eqref{eq.projection}. Then $U$ can be decomposed into closed subspaces $U = \bar U \oplus U^\perp$ where $\bar U \subset \bar{ \mathcal H}$ and $U^\perp \subset \mathcal H^\perp$. Given a function $f\in U$, let $\bar f:=\Pi_{\bar U}f$ and $f^\perp:=\Pi_{U^\perp}f$ the respective orthogonal projections of $f$. 
Note than under the present hypothesis $\bar f := \mathcal O f$ and $f^\perp := f - \mathcal O f$ since $\mathcal O$ restricted to $U$ is  the orthogonal projection to $\bar U$.

The goal is to compute $\Delta(f, \bar f)$, namely what is the excess risk of doing the regression on the function space $U$ instead of restricting to the invariant subspace $\bar U$, which corresponds to the ground truth. A simple computation shows that if $x$ is sampled from $\mu$, then
\begin{equation} \label{eq.proj}
    \Delta(f, \bar f) = \|f^\perp\|^2_\mu;
\end{equation}
this in particular shows that there is always a benefit to restricting to the ground truth subspace. For instance if the target is a group invariant function, and we do a regression in a space of polynomials, we might as well restrict the regression to group invariant polynomials, since the generalization error will be strictly smaller.

The rest of the analysis in \citet{elesedy2021provably} focuses on computing $\|f^\perp\|^2_\mu$ for data models, for instance, assuming $U$ is the space of linear functions (they discuss both over and under parameterized), and the input is Gaussian with mean zero and identity covariance (assuming the spherical Gaussian distribution is $G$-invariant). A slightly more general formulation allows them to express the same ideas for equivariance. This work has been extended to kernel regressions \citep{elesedy_kernel}.

\paragraph{Reynolds projection onto scaling invariant functions via Weyl's unitarian trick}

Let $\mathcal X = (\mathbb R, \mathbb Z^k)^d$ where each of the $d$ features $x_i$ has base units exponents $\mathbf u_i = (u_{i1}, \ldots u_{ik})\in \mathbb Z^k$ expressed in terms of $k$ base units. 
The Buckingham Pi theorem discussed in \secref{sec:approach} shows that units-equivariant functions are in a one-to-one correspondence with functions of (finitely many) dimensionless features (constructed as products of powers of input features). 
This characterization also implies that the dimensionless functions are exactly the functions that are invariant with respect to unit rescalings. For each base units, represented here by the index $j\in \{1,\ldots, k\}$, we consider $u_{ij}$, its exponent in each feature $x_i$. Then $f$ is scaling invariant (i.e. dimensionless\footnote{We assume the output is dimensionless for simplicity, without loss of generality a dimensioned output function $F^*:\mathcal Z \to \mathcal X_{\mathbf v}$ can be obtained as $F^*(\mathbf x)=g_{\mathbf x,\mathbf u}(f^*(\mathbf x))$ where $g_{\mathbf x,\mathbf u}$ is fixed.}) if and only if for all $g \in \mathbb R_{>0}$ and all $j\in \{1,\ldots, k\}$ we have
\begin{equation}
    f(g^{u_{1j}} x_1, \ldots, g^{u_{dj}} x_d ) = f(x_1, \ldots, x_d);
\end{equation}  
this property is also known as self-similarity \citep{barenblatt_1996}.
The rescalings can be seen as an action by the appropriate renormalization group, in this case $G=(\mathbb R_{>0})^k$, defined as
\begin{equation}\label{eq.group-action}
(g_1, \ldots, g_k)\cdot (x_1, \ldots, x_d) = \left( (\prod_{j=1}^kg_{j}^{-u_{1j}}) x_1, \ldots, (\prod_{j=1}^kg_{j}^{-u_{dj}}) x_d   \right).
\end{equation}
Since the group is not compact, the results of \citet{elesedy2021provably} are not directly applicable. However, it is closely related to the reductive real algebraic group $(\mathbb R^\times)^k$, so some of the concepts can be generalized by using Weyl's unitarian trick. (This will require us to work in complex space.) However, the key property from \citet{elesedy2021provably}, that the Reynolds projection coincides with the orthogonal projection with respect to the measure, will not hold in real space.


We consider $U$ a space of real analytic functions (for instance, rational monomials). We can define a Reynolds projection to the $G$-invariant functions $\bar U$, analogous to \eqref{eq.projection}, by using Weyl's unitarian trick. If $f: \mathbb R^d \to \mathbb R$ is a real analytic function it has an analytic continuation, $f_{\mathbb C}: \mathbb C^d \to \mathbb C$, such that $f_{\mathbb C}|_{\mathbb R^d} = f$. Weyl's unitarian trick is to replace the group $G = (\mathbb R_{> 0})^k$ with the torus 
$\mathbb T^k={ \{\mathbf z\in \mathbb C^k: |z_i|=1 \}}$. Both groups are Zariski-dense in the group $(\mathbb C^\times)^k$---in other words, any polynomial that vanishes identically on either $G$ or $\mathbb T^k$ actually vanishes identically on all of $(\mathbb C^\times)^k$. (For background on Weyl's trick, see \cite[pp.~129--131]{fulton2013representation}; for more on the Zariski topology and Zariski-denseness, see \cite[pp.~22--24]{shafarevich}.)  It follows, because the group action \eqref{eq.group-action} is described by rational functions, that all three groups have the same invariants. On the other hand, the torus $\mathbb T^k$ is compact, so we can average over it to obtain a projection
\begin{equation}
    \mathcal Q f(x_1, \ldots, x_d) =   \int_{\mathbb{T}^k} f_{\mathbb C} \big( (z_1, \ldots, z_k) \cdot ( x_1, \ldots, x_d )\big) d\lambda_{\mathbb C}(\mathbf z), \label{eq.proj.complex}
\end{equation}
where the Haar measure $\lambda_{\mathbb C}$ coincides with the (normalized) Lebesgue measure on the torus---namely, the Lebesgue measure scaled by $\frac{1}{(2\pi)^k}$.

In order to explain how the projection $\mathcal Q$ behaves, we first note that the rational monomials define characters for the group action. Namely,
\begin{equation}
\mathbf x^{\mathbf a} = \prod_{i=1}^d x_i^{a_i}, \quad (z_1,\ldots z_k)\cdot \mathbf x^{\mathbf a} = (\prod_{j=1}^k z_j^{\sum_{s=1}^d {u_{sj}}a_j})\mathbf x^{\mathbf a},
\end{equation}
therefore 
\begin{align}
\chi_{\mathbf x^{\mathbf a}}: \mathbb (\mathbb C_{>0})^k &\to \mathbb C_{>0} \\
(z_1, \ldots, z_k) &\mapsto \prod_{j=1}^k z_j^{\sum_{s=1}^d {u_{sj}}a_j} 
\end{align}
is a continuous group homomorphism (same for the real characters). 
Therefore, the rational monomials are eigenvectors of $\mathcal Q$:
\begin{equation}
\mathcal Q (\mathbf x^{\mathbf a}) = \int_{\mathbb T^k} 
\mathbf z \cdot \mathbf x^{\mathbf a} d\lambda_{\mathbb C}(\mathbf z)
= 
\int_{\mathbb T^k} \chi_{\mathbf x^{\mathbf a}} (\mathbf z) \, \mathbf x^{\mathbf a} d\lambda_{\mathbb C}(\mathbf z) 
= 
 \mathbf x^{\mathbf a} \int_{\mathbb T^k} \chi_{\mathbf x^{\mathbf a}} (\mathbf z) d\lambda_{\mathbb C}(\mathbf z).
\end{equation}
A standard computation shows that if $\chi$ is a character of a compact group with Haar measure $\lambda$, then $\int_G \chi(g) d\lambda = \lambda(G)$ if $\chi$ is trivial, and 0 otherwise. In particular this shows
\begin{equation}
\mathcal Q(\mathbf x^{\mathbf a}) = \left \{
\begin{matrix} 
\mathbf x^{\mathbf a} & \text{if } \sum_{s=1}^d u_{sj} a_j =0 \text{ for all } j=1,\ldots, k
\\
0 & \text{otherwise.}
\end{matrix}
\right. \label{eq.proj_monomials}
\end{equation}
 In other words, comparing \eqref{eq.proj_monomials} with \eqref{b.pi}, a rational monomial is either invariant under the group action (i.e. dimensionless), or it is in the kernel of $\mathcal Q$. Thus $\mathcal Q$ is a Reynolds operator for the group of scalings.

\paragraph{Generalization gap for (complex) units equivariant regressions} In order to use the results from \citet{elesedy2021provably} it is not enough to have a projection $\mathcal Q$. We need $\mathcal Q$ to be an orthogonal projection in an $\mathscr L_2$ space.
To this end, we consider a space of functions of the original input features, where the units equivariant functions are a linear subspace. The characterization from dimensional analysis discussed above suggests to focus on rational monomials. Unfortunately the rational monomials may have poles when features are zero, so in order to define the inner product we will restrict the measure $\mu$ to have bounded support which does not include zero. In particular we can consider $X^d=([-b,-a]\cup [a,b])^d$, and $\mu$ the standard (Lebesgue) measure in $\mathbb R^d$ restricted to $X^d$ and 0 outside $X$. We let $\mathcal H = \mathscr L_2( \mathbb R^d, \mu)$. We note that $\mu$ is not scaling invariant, and indeed, no compactly-supported measure is scaling-invariant. Relatedly, $\mathcal Q$ is not an orthogonal projection in $\mathscr L_2( \mathbb R^d, \mu)$ for any real measure $\mu$. But we will salvage what we can.

Since we extended the class of functions to a complex domain to use Weyl's trick, we also need to extend the measure $\mu$ to $\mu_{\mathbb C^d}$. We consider the Lebesgue measure in $\mathbb C^d$ with support $(X_{\mathbb C})^d$ where $X_{\mathbb C}= \{z\in \mathbb C:  a<|z|<b \}$. Now $\mathcal H_{\mathbb C} = \mathscr L_2( \mathbb C^d, \mu_{\mathbb C^d})$. We will see that even though $\mu_{\mathbb C^d}$ is a complex analog of $\mu$, the resulting Hilbert spaces are not necessarily comparable. In particular $\mathcal Q$ is an orthogonal projection in $\mathscr L_2( \mathbb C^d, \mu_{\mathbb C^d})$, due to the fact that $\mu_{\mathbb C^d}$ is rotationally symmetric (i.e., invariant under the action by $\mathbb T^d$).

\begin{proposition}
If $U \subseteq \mathcal H_{\mathbb C}$ is a linear subspace closed under scalings, then $\mathcal Q$ defined in \eqref{eq.proj.complex} is the orthogonal projection onto the space of scaling invariant functions. In particular $U=\bar U \oplus U^{\perp}$ where $\bar U = \text{Image}(\mathcal Q)$  and $U^{\perp}=\text{kernel}(\mathcal Q)$. 
\end{proposition}
\begin{proof}
It suffices to show: (a) $\mathcal Q (U)\subset U$. (b) $f$ is scaling invariant if and only if $\mathcal Q(f)=f$. (c) $\mathcal Q$ is self adjoint for $\langle \cdot, \cdot \rangle_{\mu_{\mathbb C^d}}$. This can be shown with a similar argument to \eqref{eq.sa1}-\eqref{eq.sa4}:
\begin{eqnarray}
 \langle \mathcal Q f, h \rangle_{\mu_{\mathbb C^d}} &=& \int_{X_{\mathbb C}^d} \left\langle
  \int_{\mathbb T^k } f(\mathbf z \cdot  \mathbf x) 
  d\lambda_{\mathbb C}( \mathbf z) , 
  h^*(\mathbf x) \right \rangle d\mu_{\mathbb C^d}(\mathbf x)\\
&=& 
\int_{X_{\mathbb C}^d} \int_{\mathbb T^k } \langle  f(\mathbf z \cdot \mathbf x) , h^*(\mathbf x) \rangle  d\lambda_{\mathbb C}(\mathbf z) \, 
d\mu_{\mathbb C^d}(\mathbf x) \\
&=& \int_{X_{\mathbb C}^d} \int_{\mathbb T^k } \langle  f( \mathbf x) , 
h^*(\mathbf z^{-1}\cdot \mathbf x) \rangle  d\lambda_{\mathbb C}(\mathbf z) \, d\mu_{\mathbb C^d}(\mathbf x) \label{eq.ginverse2}\\
&=& \langle f,  \mathcal Q h \rangle_{\mu_{\mathbb C^d}},
\end{eqnarray}
where $h^*$ is the complex conjugate of $h$.
Note that \eqref{eq.ginverse2} holds because the measure $\mu_{\mathbb C^d}$ is invariant with respect to scalings in $\mathbb T^k$ (namely, $\mu_{\mathbb C^d}$ is rotationally symmetric). 
\end{proof}

Note this argument is not possible for $\mu$. We'll see at the end of this section that $\mathcal Q$ is not self-adjoint with respect to any real inner product in general.

Our argument in the previous section shows that when $U$ is a space generated by rational monomials, the projection onto $\bar U$ is easy to characterize. In particular, a simple computation (Proposition \eqref{eq.orthogonal}) shows that the complex rational monomials are orthogonal in $\mathscr L_2( \mathbb C^d, \mu_{\mathbb C^d})$.
\begin{proposition} \label{eq.orthogonal}
Let $U$ be the space of rational monomials, spanned by
\begin{equation}
\mathcal B:=\left\{\mathbf x^{\mathbf a}:= \prod_{i=1}^d x_i^{a_i}: \mathbf a=(a_1,\ldots , a_d) \in \mathbb Z^d\right\}.
\end{equation} 
Then  for all  $\mathbf x^{\mathbf a}, \mathbf x^{\mathbf a'} \in \mathcal B$ with $\mathbf a \neq \mathbf a'$ we have $\langle\mathbf x^{\mathbf a}, \mathbf x^{\mathbf a'}\rangle_\mu = 0$.
\end{proposition}
\begin{proof} Let $\mathbf a \neq \mathbf a'$, then
\begin{align}
 \langle\mathbf x^{\mathbf a}, \mathbf x^{\mathbf a'}\rangle_{\mu_{\mathbb C}} &= \int_{ X_{\mathbb C^d} } \prod_{i=1}^d x_i^{a_{i}} \bar{x}_i^{a'_{i}} d\mu_{\mathbb C^d} \\
&= \int_{X_{\mathbb C^d}}\prod_{i=1}^d r_ie^{j a_i \theta_i} r_ie^{- j a'_i \theta_i}  d\mu_{\mathbb C^d},
\end{align}
where $x_i=e^{j \theta_i}$ and $j=\sqrt{-1}$. Now since $\mathbf a\neq \mathbf a'$ we can choose $s$ such that $a_s\neq a'_s$. Using Fubini-Tonelli's theorem we can write:
\begin{align}
\langle\mathbf x^{\mathbf a}, \mathbf x^{\mathbf a'}\rangle_{\mu_{\mathbb C}}
&= \int_{ X_{\mathbb C^{d-1}}}  \prod_{\substack{i=1\\ i\neq s}}^d r_i e^{j (a_i - a_i') \theta_i}  d\mu_{\mathbb C^{d-1}} \int_{X_{\mathbb C}} r_s e^{j (a_s-a'_s )\theta_s} d\mu_{\mathbb C^1} ,
\end{align}
where the last term is zero because $a_s\neq a'_s$ and the measure $\mu_{\mathbb C^1}$ is rotationally symmetric. 
\end{proof}


This setting allows to generalize the results from \citet{elesedy2021provably} to complex scalings of complex functions.
\begin{proposition}
Let $X\sim \mu_{\mathbb C^d}$ where $\mu_{\mathbb C^d}$ is a rotation invariant distribution in $A \subset \mathbb C^d$. Let $Y=f^*(X) + \xi \in \mathbb C$, where $\xi$ is a random element of $\mathbb C$ that is independent of $X$ with zero mean and finite variance, and $f^*: A \to \mathbb C$ is scaling invariant. Then, for any $f$, the generalization gap satisfies
\begin{equation}
\Delta(f, \mathcal Q f) = \|f^\perp \|_{\mu_{\mathbb C^d}}^2.
\end{equation}
\end{proposition}

\paragraph{Discussion of real units-equivariant functions}
Even though some dimensional quantities can be complex, for example electromagnetic field amplitudes in Fourier space, most dimensional quantities are real-valued, and the dimensional scalings are always real, therefore the above theory is not directly applicable.

Unfortunately the analysis above will not hold for the real case. Note that the Reynolds operator defined in \eqref{eq.proj.complex} is well-defined for real-analytic functions and delivers the same projection in the rational monomials case (i.e. it drops the non-dimensionless rational monomials). However this projection does not correspond to an orthogonal projection with respect to any nontrivial measure $\mu$. One way to see this is by observing that the monomials cannot be orthogonal for any non-trivial real measure $\mu$. 

For example, suppose we have two input features $x$, $y$ and the dimensionless space is generated by $xy$. Then the Reynolds projection $\mathcal Q$ will satisfy that $\mathcal Q (xy^{-1})=0$ (see Eq. \eqref{eq.proj_monomials}). 
However, no non-trivial real measure will satisfy that $\langle xy, xy^{-1} \rangle_\mu =0$ because $\langle xy, xy^{-1} \rangle_\mu = \int (xy) (xy^{-1}) d\mu = \int x^2 d\mu = \langle x, x \rangle_\mu >0$. 

The underlying question here is, what is the right notion of projection of a real function onto the space of units-equivariant functions?
The Reynolds projection is the most natural projection from an algebraic point of view, per the discussion following \eqref{eq.projection}.
The orthogonal projection with respect to the $\mathscr L^2$-norm of the measure of the data is the one corresponding to the estimator risk \eqref{eq.risk}. The fact that these two projections diverge in the present case stems from the fact that the measure from which the data is drawn cannot itself be scaling-invariant. Furthermore, there is no reason to expect that the space spanned by rational monomials is closed under the orthogonal projection in $\mathscr L^2(\mathbb R^d, \mu)$.

Note that our algorithm does not perform a projection, it directly optimizes in a space of invariant functions. The generalization gap one may want to investigate is $\Delta(\hat f, \bar f)$ where $\hat f$ is the output of a baseline regression, and $\bar f$ is our units-equivariant regression. We show specific examples in \secref{sec:experiments}.



\end{document}